\documentclass[twoside,11pt]{article}
 
\usepackage{jmlr2e,times}
 \usepackage{natbib,color}

 \usepackage[colorlinks,
            linkcolor=blue,
            citecolor=blue,
            urlcolor=magenta,
            linktocpage,
            plainpages=false]{hyperref}

\newcommand{\BEAS}{\begin{eqnarray*}}
\newcommand{\EEAS}{\end{eqnarray*}}
\newcommand{\BEA}{\begin{eqnarray}}
\newcommand{\EEA}{\end{eqnarray}}
\newcommand{\BEQ}{\begin{equation}}
\newcommand{\EEQ}{\end{equation}}
\newcommand{\BIT}{\begin{itemize}}
\newcommand{\EIT}{\end{itemize}}
\newcommand{\BNUM}{\begin{enumerate}}
\newcommand{\ENUM}{\end{enumerate}}
\newcommand{\BA}{\begin{array}}
\newcommand{\EA}{\end{array}}
\newcommand{\diag}{\mathop{\rm diag}}

\newcommand{\tr}{\mathop{ \rm tr}}

\newcommand{\idm}{I}
\newcommand{\rb}{\mathbb{R}}
\newcommand{\Sb}{\mathbb{S}}

\newcommand{\mysec}[1]{Section~\ref{sec:#1}}
\newcommand{\eq}[1]{Eq.~(\ref{eq:#1})}
\newcommand{\myfig}[1]{Figure~\ref{fig:#1}}

\def \E{{\mathbb E}}

 \def  \V{ \mathcal{V}}
 \def  \Y{ \mathcal{Y}}
  \def  \X{ \mathcal{X}}
 \def  \F{ \mathcal{F}}
 \def  \G{ \mathcal{G}}

\ShortHeadings{Breaking the Curse of Dimensionality with Convex Neural Networks}{Bach}
\firstpageno{1}

\begin{document}

\title{Breaking the Curse of Dimensionality with Convex Neural Networks}

\author{\name Francis Bach \email francis.bach@ens.fr \\
       \addr INRIA - Sierra Project-team\\
       D\'epartement d'Informatique de l'Ecole Normale Sup\'erieure \\
     Paris, France
     }

\editor{}

\maketitle
\begin{abstract}
We consider neural networks with a single hidden layer and non-decreasing positively homogeneous activation functions like the rectified linear units. By letting the number of hidden units grow unbounded and using classical non-Euclidean regularization tools on the output weights, they lead to a convex optimization problem and we  provide a detailed theoretical analysis of their generalization performance, with a study of both the approximation  and the estimation errors. We show in particular that they are adaptive to unknown underlying linear structures, such as the dependence on the projection of the input variables onto a low-dimensional subspace. Moreover, when using sparsity-inducing norms on the input weights, we show that high-dimensional non-linear variable selection may be achieved, without any strong assumption regarding the data and with a total number of variables potentially exponential in the number of observations. However, solving this convex optimization problem in infinite dimensions is only possible if the  non-convex subproblem of  addition of a new unit can be solved efficiently. We 
 provide a simple geometric interpretation for our choice of activation functions and 
 describe simple conditions for convex relaxations of the finite-dimensional non-convex subproblem to achieve the same generalization error bounds, even when constant-factor approximations cannot be found. We were not able to find strong enough convex relaxations to obtain provably polynomial-time algorithms and leave open the existence or non-existence of such tractable algorithms with non-exponential sample complexities.
\end{abstract}


\section{Introduction}
\label{sec:intro}

Supervised learning methods come in a variety of ways. They are typically based on local averaging methods, such as $k$-nearest neighbors, decision trees, or random forests,  or on optimization of the empirical risk over a certain function class, such as least-squares regression, logistic regression or support vector machine, with positive definite kernels, with model selection, structured sparsity-inducing regularization, or boosting~\citep[see, e.g.,][and references therein]{gyorfi2002distribution,hastie2009,shaibook}.

Most methods assume either explicitly or implicitly a certain class of models to learn from. In the non-parametric setting, the learning algorithms may adapt the complexity of the models as the number of observations increases: the sample complexity (i.e., the number of observations) to adapt to any particular problem is typically large. For example, when learning Lipschitz-continuous functions in $\rb^d$, at least $n = \Omega(\varepsilon^{-\max\{d,2\}})$ samples are needed to learn a function with excess risk~$\varepsilon$~\citep[][Theorem 15]{luxburg2004distance}. The exponential dependence on the dimension $d$ is often referred to as the \emph{curse of dimensionality}: without any restrictions, exponentially many observations are needed to obtain optimal generalization performances.

At the other end of the spectrum, parametric methods such as linear supervised learning make strong assumptions regarding the problem and  generalization bounds based on estimation errors typically assume that the model is well-specified, and the sample complexity to attain an excess risk of $\varepsilon$ grows as $n = \Omega( d/\varepsilon^2)$, for linear functions in $d$ dimensions and Lipschitz-continuous loss functions~\citep[][Chapter 9]{shaibook}. While the sample complexity is much lower, when the assumptions are not met,  the methods  underfit and more complex models would provide better generalization performances.

Between these two extremes, there are a variety of models with structural assumptions that are often used in practice. For input data in $x \in \rb^d$, prediction functions $f: \rb^d \to \rb$ may for example be parameterized as:
\BIT
\item[(a)] \emph{Affine functions}: $f(x) = w^\top x +b $, leading to potential severe underfitting, but easy optimization and good (i.e., non exponential) sample complexity.

\item[(b)] \emph{Generalized additive models}: $f(x) = \sum_{j=1}^d f_j(x_j)$, which are generalizations of the above by summing functions $f_j: \rb \to \rb$ which may not be affine~\citep{hastie_GAM,spam,grouplasso}. This leads to less strong underfitting but cannot model interactions between variables, while the estimation may be done with similar tools than for affine functions (e.g., convex optimization for convex losses).

\item[(c)] \emph{Nonparametric ANOVA models}: $f(x) = \sum_{A \in \mathcal{A} } f_A(x_A)$ for a set $\mathcal{A}$ of subsets of $\{1,\dots,d\}$, and non-linear functions $f_A:\rb^A \to \rb$. The set $\mathcal{A}$ may be either given~\citep{gu2013smoothing} or learned from data~\citep{cosso,hkl}. Multi-way interactions are explicitly included but a key algorithmic problem is to explore the $2^d-1$ non-trivial potential subsets.

\item[(d)] \emph{Single hidden-layer neural networks}: $f(x) = \sum_{j=1}^k \sigma(w_j^\top x + b_j)$, where $k$ is the number of units in the hidden layer~\citep[see, e.g.,][]{rumelhart1986learning,haykin1994neural}. The activation function $\sigma$ is here assumed to be fixed. While the learning problem may be cast as a (sub)differentiable optimization problem, techniques based on gradient descent may not find the global optimum. If the number of hidden units is fixed, this is a parametric problem.

\item[(e)] \emph{Projection pursuit}~\citep{friedman1981projection}: $f(x) = \sum_{j=1}^k f_j(w_j^\top x )$ where $k$ is the number of projections. This model combines both (b) and (d); the only difference with neural networks is that the non-linear functions $f_j:\rb \to \rb$ are learned from data. The optimization is often done sequentially and is harder than for neural networks.

\item[(e)] \emph{Dependence on a unknown $k$-dimensional subspace}: $f(x) = g(W^\top x  )$ with $W \in \rb^{d \times k}$, where $g$ is a non-linear function. A variety of algorithms exist for this problem \citep{li1991sliced,fukumizu2004dimensionality,DalalyanJS08}. Note that when the columns of $W$ are assumed to be composed of a single non-zero element, this corresponds to \emph{variable selection} (with at most $k$ selected variables).
\EIT

In this paper, our main aim is to answer the following question: \textbf{Is there a \emph{single} learning method that can deal \emph{efficiently}     with all situations above with \emph{provable adaptivity}?} We consider single-hidden-layer neural networks, with non-decreasing homogeneous activation functions such as 
$$\sigma(u) = \max\{u,0\}^\alpha = (u)_+^\alpha,$$ 
for $\alpha \in \{0,1,\dots\}$, with a particular focus on $\alpha = 0$ (with the convention that $0^0=0$), that is $\sigma(u) = 1_{u >0} $ (a threshold at zero), and $\alpha=1$, that is, $\sigma(u) = \max\{u,0\} = (u)_+$, the so-called \emph{rectified linear unit}~\citep{nair2010rectified,krizhevsky2012imagenet}. We follow the convexification approach of~\citet{bengio2006convex,rosset2007L1}, who consider potentially infinitely many units and let a sparsity-inducing norm choose the number of units automatically.  This leads naturally to incremental algorithms such as forward greedy selection approaches, which have a long history for single-hidden-layer neural networks~\citep[see, e.g.][]{breiman1993hinging,lee1996efficient}.

We make the following contributions:

\vspace*{-.125cm}

\BIT

\item[--] We provide in \mysec{general} a review of functional analysis tools used for learning from continuously infinitely many basis functions, by studying carefully the similarities and differences between $L_1$- and $L_2$-penalties on the   output weights. For $L_2$-penalties, this corresponds to a positive definite kernel and may be interpreted through random sampling of hidden weights.
We also review incremental algorithms (i.e., forward greedy approaches) to learn from these infinite sets of  basis functions when using $L_1$-penalties.

\item[--] The results are specialized in \mysec{neural} to neural networks with a single hidden layer and activation functions which are  positively homogeneous (such as the rectified linear unit). In particular, in Sections~\ref{sec:FW1}, \ref{sec:FW2} and \ref{sec:FW3}, we provide simple geometric interpretations to the non-convex problems of  additions of new units, in terms of separating hyperplanes or Hausdorff distance between convex sets. They constitute the core potentially hard computational tasks in our framework of learning from continuously many basis functions.

\item[--] In \mysec{approx}, we provide a detailed theoretical analysis of the approximation properties of (single hidden layer) convex neural networks with monotonic homogeneous activation functions, with explicit bounds. We relate these new results to the extensive literature on approximation properties of neural networks~\citep[see, e.g.,][and references therein]{pinkus1999approximation} in \mysec{related}, and show that these neural networks are indeed adaptive to linear structures, by replacing the exponential dependence in dimension by an exponential dependence in the dimension of the subspace of the data can be projected to for good predictions.

\item[--] In \mysec{adaptivity}, we study the generalization properties under a standard supervised learning set-up, and show that these convex neural networks are adaptive to all situations mentioned earlier. These are summarized in Table~\ref{tab:summary} and constitute the main statistical results of this paper.
 When using an $\ell_1$-norm on the input weights, we show in \mysec{highdim} that high-dimensional non-linear variable selection may be achieved, that is, the number of input variables may be much larger than the number of observations, without any strong assumption regarding the data (note that we do not present a polynomial-time algorithm to achieve this).

\item[--] We provide in \mysec{sufficient} simple conditions for convex relaxations to achieve the same generalization error bounds, even when constant-factor approximation cannot be found (e.g., because it is NP-hard such as for the threshold activation function and the rectified linear unit). We present in \mysec{gauge} convex relaxations based on semi-definite programming, but we were not able to find strong enough convex relaxations (they provide only a provable sample complexity with a polynomial time algorithm 
  which is exponential in the dimension $d$) and leave open the existence or non-existence of polynomial-time algorithms that preserve the non-exponential sample complexity.
\EIT

\begin{table}

\begin{center}
\begin{tabular}{|ll|c|}
\hline
\textcolor{white}{$ \Big|$}   & \hspace*{.8cm} Functional form & Generalization bound        \\
\hline
 No assumption & & 
\textcolor{white}{$ \bigg|$}  $ 
       {n^{- 1/( d+3)}}\log n $ 
     \\
     \hline
Affine function & 
\textcolor{white}{$ \bigg|$}   $w^\top x + b$  \textcolor{white}{$ \bigg|$} &   $   d^{1/2}
\cdot    {n^{-1/2}}$   \\
\hline
Generalized additive model & 
\textcolor{white}{$ \bigg|$}  $  \sum_{j=1}^k f_j( w_j^\top x)$, $w_j \in \rb^d$ &
$ 
  k   d^{1/2} \cdot  {n^{-1/4}} \log n $ 
   \\
\hline
Single-layer neural network & \textcolor{white}{$ \bigg|$}  $ \sum_{j=1}^k \eta_j  ( w_j ^\top x + b_j)_+ $ & 
$ k  d^{1/2}
\cdot    {n^{-1/2}}$  \\
\hline
Projection pursuit& 
\textcolor{white}{$ \bigg|$}  $  \sum_{j=1}^k f_j( w_j^\top x)$, $w_j \in \rb^d$ &
$   k d^{1/2}
\cdot    {n^{-1/4}} \log n$ 
   \\
\hline
Dependence on subspace & 
\textcolor{white}{$ \bigg|$}  $  f( W^\top x)$ , $W \in \rb^{d \times s}$  &
$     d^{1/2}  \cdot    {n^{-1/( s+3)}} \log n$ 
   \\
\hline
\end{tabular}
 \end{center}

\caption{Summary of generalization bounds for various models. The bound represents the expected excess risk over the best predictor in the given class. When no assumption is made, the dependence in $n$ goes to zero with an exponent proportional to $1/d$ (which leads to sample complexity exponential in $d$), while making assumptions removes the dependence of $d$ in the exponent.
\label{tab:summary}}
\end{table}

\section{Learning from continuously infinitely many basis functions}
\label{sec:general}

In this section we present the functional analysis framework underpinning the methods presented in this paper, which learn for a potential continuum of features. While the formulation from Sections~\ref{sec:variation} and~\ref{sec:rep} originates from the early work on the approximation properties of neural networks~\citep{barron1993universal,kurkova,mhaskar2004tractability}, the algorithmic parts that we present in \mysec{conditional} have been studied in a variety of contexts, such as ``convex neural networks''~\citep{bengio2006convex}, or $\ell_1$-norm with infinite dimensional feature spaces~\citep{rosset2007L1}, with links with conditional gradient algorithms~\citep{dunn1978conditional,jaggi} and boosting~\citep{rosset2004boosting}.

In the following sections, note that there will be two different notions of \emph{infinity}:  infinitely many inputs $x$  and infinitely many basis functions $x \mapsto \varphi_v(x)$. Moreover, two orthogonal notions of \emph{Lipschitz-continuity} will be tackled in this paper: the one of the prediction functions $f$, and the one of the loss $\ell$ used to measure the fit of these prediction functions.

\subsection{Variation norm}
\label{sec:variation}
We consider an arbitrary measurable input space $\mathcal{X}$ (this will a sphere in $\rb^{d+1}$ starting from \mysec{neural}), with a set of \emph{basis functions} (a.k.a.~\emph{neurons} or \emph{units}) $\varphi_v : \mathcal{X} \to \rb$, which are parameterized by $v \in \V$, where $\V$ is a compact topological space (typically a  sphere for a certain norm on $\rb^d$ starting from \mysec{neural}). We assume that for any given $x \in \X$, the functions $v \mapsto \varphi_v(x)$ are continuous. These functions will be the hidden neurons in a single-hidden-layer neural network, and thus $\V$ will be $(d+1)$-dimensional for inputs of dimension $d$ (to represent any affine function). Throughout \mysec{general}, these features will be left unspecified as most of the tools apply more generally.

In order to define our space of functions from $\mathcal{X} \to \rb$, we need real-valued Radon measures, which are continuous linear forms on the space of continuous functions from $\mathcal{V}$ to $\rb$, equipped with the uniform norm~\citep{rudin1987real,evans1991measure}. For a continuous function $g: \mathcal{V} \to \rb$ and a Radon measure $\mu$, we will use the standard notation
$\int_{\V} g(v) d \mu(v)$ to denote the action of the measure $\mu$ on the continuous function $g$. The norm of $\mu$ is usually referred to as its \emph{total variation} (such finite total variation corresponds to having a continuous linear form on the space of continuous functions), and we denote it as $|\mu|( \V)$, and is equal to the supremum of $\int_{\V} g(v) d \mu(v)$ over all continuous functions with values in $[-1,1]$.
As seen below, when $\mu$ has a density with respect to a probability measure, this is the $L_1$-norm of the density.

 We consider the space $\F_1$ of functions $ f   $ that can be written as 
$$
f(x) = \int_{\V} \varphi_v(x) d \mu(v),
$$
where $\mu$ is a signed Radon measure on $\V$ with finite total variation $|\mu|( \V)$. 

When $\V$ is finite, this corresponds to
$$f(x) = \sum_{v \in \V }  \mu_v \varphi_{v}(x),$$
with total variation $\sum_{v \in \V} |\mu_v|$, where the proper formalization for infinite sets $\V$ is done through measure theory.

The infimum of   $|\mu|( \V)$ over all decompositions of $f$ as $f = \int_{\V} \varphi_v d \mu(v)$, turns out to be a  norm~$\gamma_1$ on~$\F_1$, often called the \emph{variation} norm of $f$ with respect to the set of basis functions \citep[see, e.g.,][]{kurkova,mhaskar2004tractability}. 

Given our assumptions regarding the compactness of $\V$, for any $f \in \F_1$, the infimum defining $\gamma_1(f)$ is in fact attained by a signed measure $\mu$, as a consequence of the compactness of measures for the weak topology~\citep[see][Section 1.9]{evans1991measure}.

In the definition above, if we assume that the signed measure $\mu$ has a density with respect to a fixed \emph{probability} measure $\tau$ with full support on $\mathcal{V}$, that is, $d \mu(v) = {p}(v) d \tau (v)$, then, the   variation norm $\gamma_1(f)$ is  also equal to the infimal value of
$$|\mu|( \V) = \int_\V  |{p}(v)| d \tau(v),$$
over all integrable functions $p$ such that  $f(x) = \int_\V {p}(v) \varphi_v(x) d \tau(v)$.  Note however that not all measures have densities, and that the two infimums are the same as all Radon measures are limits of measures with densities. Moreover, the infimum in the definition above is not attained in general (for example when the optimal measure is singular with respect to $d \tau$); however, it often provides a more intuitive definition of the variation norm, and leads to easier comparisons with Hilbert spaces in \mysec{rkhs}.

\paragraph{Finite number of neurons.} If $f: \X \to \rb$ is decomposable into $k$ basis functions, that is, 
$f(x) = \sum_{j=1}^k \eta_j \varphi_{v_j}(x)$, then this corresponds to $\mu = \sum_{j=1}^k \eta_j \delta(v=v_j)$, and the total variation of $\mu$ is equal to the $\ell_1$-norm $\| \eta \|_1$ of $\eta$. Thus the function $f$ has variation norm less than  $\| \eta\|_1$ or equal. This is to be contrasted with the number of basis functions, which is the $\ell_0$-pseudo-norm of $\eta$.

\subsection{Representation from finitely many functions}
\label{sec:rep}

When minimizing any functional $J$ that depends only on the function values taken at a subset $\hat{\X}$ of values in $\X$,  over the ball $\{ f \in \F_1, \ \gamma_1(f) \leqslant \delta\}$, then we have a ``representer theorem'' similar to the reproducing kernel Hilbert space situation, but also with significant differences, which we now present.

 The problem is indeed simply equivalent to minimizing a functional on  functions restricted to $\X$, that is, to minimizing 
$J(f_{|\hat{\mathcal{X}}})$ over $f_{|\hat{\X}} \in \rb^{\hat{\X}}$, such that $ \gamma_{1|\hat{\X}}(f_{|\hat{\X}}) \leqslant \delta$, where
$$
 {\gamma}_{1|\hat{\X}}( f_{|\hat{\X}} ) = \inf_\mu |\mu|(\V) \mbox{ such that }  \forall x \in \hat{\X}, \   f_{|\hat{\X}}(x)  = \int_{\V}{\varphi}_{v}(x) d\mu(v);
$$
we can then build a function defined over all $\X$, through the optimal measure $\mu$ above. 

Moreover,  by Carath\'eodory's theorem for cones~\citep{rockafellar97}, if $\hat{\X}$ is composed of only~$n$ elements (e.g., $n$ is the number of observations in machine learning),  the optimal function $ f_{|\hat{\X}} $ above  (and hence $f$) may be decomposed into at most~$n$ functions $\varphi_v$, that is, $\mu$ is supported by at most $n$ points in $\V$, among a potential continuum of possibilities.

 Note however that the identity of these $n$ functions is not known in advance, and thus there is a significant difference with the representer theorem for positive definite kernels and Hilbert spaces~\citep[see, e.g.,][]{Cristianini2004}, where the set of $n$ functions are known from the knowledge of the points $x \in \hat{\X}$ (i.e., kernel functions evaluated at $x$).

\subsection{Corresponding reproducing kernel Hilbert space (RKHS)}
\label{sec:RKHS}
\label{sec:rkhs}

We have seen above that if the real-valued measures $\mu$ are restricted to have density $p$ with respect to a fixed probability measure $\tau$ with full support on $\mathcal{V}$, that is, $d \mu(v) = {p}(v) d \tau (v)$, then, 
the norm $\gamma_1(f)$ is the infimum of the total variation 
$|\mu|( \V)= \int_\V  |{p}(v)| d \tau(v),$
over all decompositions $f(x) = \int_\V {p}(v) \varphi_v(x) d \tau(v)$.

We may  also define the infimum of  $ \int_\V  |{p}(v)|^2 d \tau(v)$ over the same decompositions (squared $L_2$-norm instead of $L_1$-norm).
 It turns out that it defines a squared norm $\gamma_2^2$ and that the function space~$\F_2$ of functions with finite norm happens to be a reproducing kernel Hilbert space (RKHS). When $\V$ is finite, then it is well-known~\citep[see, e.g.,][Section~4.1]{berlinet2004reproducing} that the infimum of $\sum_{v \in \V} \mu_v^2$ over all vectors $\mu$ such that $f = \sum_{v \in V} \mu_v \varphi_v$ defines a squared RKHS norm with positive definite kernel $k(x,y) = \sum_{v \in V} \varphi_v(x) \varphi_v(y)$.
 
 We show in Appendix~\ref{app:rkhs} that for any compact set $\V$, we have defined a squared RKHS norm 
  $\gamma_2^2$  with positive definite kernel $ \displaystyle k(x,y) = \int_\V  \varphi_v(x) \varphi_v(y) d \tau(v)$.
 
 \paragraph{Random sampling.}
 Note that such kernels are well-adapted to approximations by sampling several basis functions $\varphi_v$ sampled from the probability measure $\tau$~\citep{neal1995bayesian,rahimi2007random}. Indeed, if we consider $m$ i.i.d.~samples $v_1,\dots,v_m$, we may define the approximation
 $\hat{k}(x,y) = \frac{1}{m} \sum_{i=1}^m \varphi_{v_i}(x) \varphi_{v_i}(y)$, which  corresponds to an explicit feature representation. In other words, this corresponds to sampling units  $v_i$, using prediction functions of the form $\frac{1}{m} \sum_{i=1}^m \eta_i \varphi_{v_i}(x)$ and then penalizing by the $\ell_2$-norm of $\eta$. 
 
 When $m$ tends to infinity, then $\hat{k}(x,y)$ tends to $k(x,y)$ and random sampling provides a way to work efficiently with explicit $m$-dimensional feature spaces. See~\citet{rahimi2007random} for a analysis of the number of units needed for an approximation with error $\varepsilon$, typically of order $1/\varepsilon^2$. See also~\citet{kernelexp} for improved results with a better dependence on $\varepsilon$ when making extra assumptions on the eigenvalues of the associated covariance operator.

 \paragraph{Relationship between $\F_1$ and $\F_2$.}
 
 The corresponding RKHS norm is always greater than the variation norm (because of Jensen's inequality), and thus the RKHS $\F_2$ is included in $\mathcal{F}_1$. However, as shown in this paper, the two spaces $\F_1$ and $\F_2$ have very different properties; e.g., $\gamma_2$ may be computed easily in several cases, while $\gamma_1$ does not; also, learning with $\F_2$ may either be done by random sampling of sufficiently many weights or using kernel methods, while $\F_1$ requires dedicated convex optimization algorithms with potentially non-polynomial-time steps (see~\mysec{cg}).
 
  Moreover, for any $v \in \V$, $\varphi_v \in \F_1$ with a norm $\gamma_1(\varphi_v) \leqslant 1$, while in general $\varphi_v \notin \F_2$. This is a simple illustration of the fact that $\F_2$ is too small and thus will lead to a  lack of adaptivity that will be further studied in \mysec{comp} for neural networks with certain activation functions.

\subsection{Supervised machine learning}
\label{sec:learning-losses}

Given some  distribution over the pairs $(x,y) \in \X\times \Y$, a loss function $\ell: \Y \times \rb \to \rb$, our aim is to find a function $f: \X\to \rb$ such that the functional $J(f) = \E \big[ \ell(y,f(x)) \big]$ is small, given some i.i.d.~observations $(x_i,y_i)$, $i=1,\dots,n$. We consider the empirical risk minimization framework over a space of functions $\F$, equipped with a norm $\gamma$ (in our situation, $\F_1$ and $\F_2$, equipped with  $\gamma_1$ or $\gamma_2$). The empirical risk $\hat{J}(f) = \frac{1}{n} \sum_{i=1}^n \ell(y_i,f(x_i))$, is minimized either (a) by constraining $f$ to be in the ball $\mathcal{F}^\delta = \{ f \in \F, \ \gamma(f) \leqslant \delta \}$ or (b) regularizing the empirical risk by $\lambda \gamma(f)$. 
Since this paper has a more theoretical nature, we focus on constraining, noting that in practice, penalizing is often more robust~\citep[see, e.g.,][]{harchaoui2013conditional} and leaving its analysis in terms of learning rates for future work. Since the functional $\hat{J}$ depends only on function values taken at finitely many points, the results from \mysec{rep} apply and we expect the solution $f$ to be spanned by only $n$ functions $\varphi_{v_1},\dots,\varphi_{v_n}$ (but we ignore in advance which ones among all $\varphi_v$, $v \in \V$, and the algorithms in \mysec{cg} will provide approximate such representations with potentially less or more than $n$ functions).

\paragraph{Approximation error vs. estimation error.}
We consider an $\varepsilon$-approximate minimizer of $\hat{J}(f) = \frac{1}{n} \sum_{i=1}^n \ell(y_i,f(x_i))$ on the convex set $\mathcal{F}^\delta$, that is a certain $\hat{f} \in \F^\delta$ such that
$\hat{J}(\hat{f}) \leqslant \varepsilon + \inf_{f \in \F^\delta} \hat{J}(f)$. We thus have, using standard arguments~\citep[see, e.g.,][]{shaibook}:
$$
J(\hat{f}) - \inf_{f\in \F} J(f)
\leqslant \bigg[ \inf_{f\in \F^\delta } J(f) - \inf_{f\in \F} J(f) \bigg] + 2 \sup_{ f \in \F^\delta} | \hat{J}(f) - J(f) | + \varepsilon,
$$
that is, the excess risk $J(\hat{f}) - \inf_{f\in \F} J(f)$ is upper-bounded by a sum of an \emph{approximation error}  $\inf_{f\in \F^\delta } J(f) - \inf_{f\in \F} J(f)$, an \emph{estimation error} 
$ 2 \sup_{ f \in \F^\delta} | \hat{J}(f) - J(f) | $ and an \emph{optimization error} $\varepsilon$ \citep[see also][]{bottou-bousquet-2008b}.
In this paper, we will deal with all three errors, starting from the optimization error which we now consider for the space $\F_1$ and its variation norm.

\subsection{Incremental conditional gradient algorithms}
\label{sec:condgrad}
\label{sec:cg}
\label{sec:conditional}

In this section, we review algorithms to minimize a smooth functional $J: L_2(d\rho) \to \rb$, where~$\rho$ is a probability measure on $\X$. This may typically be the expected risk or the empirical risk above. 
When minimizing $J(f)$ with respect to $f \in \mathcal{F}_1$ such that $\gamma_1(f) \leqslant \delta$, we need algorithms that can efficiently optimize a convex function over an infinite-dimensional space of functions. Conditional gradient algorithms allow to incrementally build a set of elements of $\mathcal{F}_1^\delta = \{ f \in \F_1, \ \gamma_1(f) \leqslant \delta\}$; see, e.g.,~\citet{frank2006algorithm,dem1967minimization,dudik2012lifted,harchaoui2013conditional,jaggi,bach2012duality}.

 \begin{figure}
\begin{center}
 \includegraphics[scale=1]{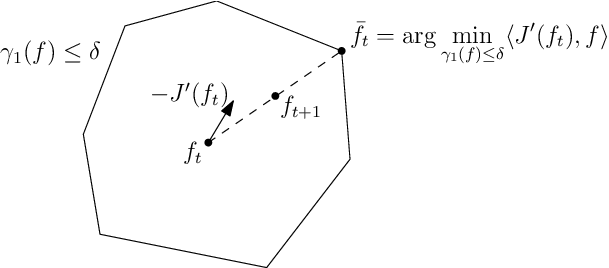}
 \end{center}
 
 \vspace*{-.5cm}
 
 \caption{Conditional gradient algorithm for minimizing a smooth functional $J$ on  $\mathcal{F}_1^\delta = \{ f \in \F_1, \ \gamma_1(f) \leqslant \delta\}$: going from $f_t$ to $f_{t+1}$; see text for details.}
 \label{fig:cg}
 \end{figure}

\paragraph{Conditional gradient algorithm.}
We assume the functional $J$ is convex and $L$-smooth, that is for all $h \in L_2(d\rho)$, there exists a gradient $J'(h) \in L_2(d\rho)$ such that for all $f \in L_2(d\rho)$, 
$$
0 \leqslant J(f)  - J(h) - \langle f - h, J'(h) \rangle_{L_2(d\rho)} \leqslant \frac{L}{2} \| f - h\|^2_{L_2(d\rho)}.
$$
When $\X$ is finite, this corresponds to the regular notion of smoothness from convex optimization~\citep{nesterov2004introductory}.

The conditional gradient algorithm (a.k.a.~Frank-Wolfe algorithm) is an iterative algorithm, starting from any function $f_0 \in \F_1^\delta$ and with the following  recursion, for $t \geqslant 0$:
\BEAS
\bar{f}_t & \in &   \arg\min_{ f \in \F_1^\delta}  \ \langle f, J'(f_t) \rangle_{ L_2(d\rho)}\\
f_{t+1} &  = & ( 1- \rho_t) f_t + \rho_t \bar{f}_t.
\EEAS
See an illustration in \myfig{cg}.
We may choose either $\rho_t = \frac{2}{t+1}$ or perform a line search for $\rho_t \in [0,1]$. For all of these strategies, the $t$-th iterate is a convex combination of the functions $\bar{f}_0,\dots,\bar{f}_{t-1}$, and is thus an element of $\F_1^\delta$. It is known that for these two strategies for $\rho_t$, we have the following convergence rate~\citep[see, e.g.][]{jaggi}:  
$$
J(f_t) - \inf_{ f \in \F_1^\delta} J(f) \leqslant \frac{2L}{t+1} \sup_{f,g \in   \F_1^\delta} \| f - g\|^2_{ L_2(d\rho)}.
$$
When, $r^2 = \sup_{ v \in \V} \| \varphi_v \|^2_{ L_2(d\rho)}$ is finite, we have
$\| f\|^2_{ L_2(d\rho)} \leqslant r^2 \gamma_1(f)^2$ and thus we get a convergence rate of $\frac{2Lr^2 \delta^2}{t+1}$.

Moreover, the basic Frank-Wolfe (FW) algorithm may be extended to handle the regularized problem as well \citep{harchaoui2013conditional,siammatrix,zhang2012accelerated}, with similar convergence rates in $O(1/t)$. Also, the second step in the algorithm, where the function $f_{t+1}$ is built in the segment between $f_t$ and the newly found extreme function, may be replaced by the optimization of  $ {J}$ over the convex hull of all functions $\bar{f}_0, \dots, \bar{f}_{t}$, a variant which is often referred to as \emph{fully corrective}. Moreover, in our context where $\V$ is a space where local search techniques may be considered, there is also the possibility of ``fine-tuning'' the vectors $v$ as well~\citep{bengio2006convex}, that is,  we may optimize the function $(v_1,\dots,v_t,\alpha_1,\dots,\alpha_t) \mapsto J\big(
\sum_{i=1}^t \alpha_i \varphi_{v_i}
\big)$, through local search techniques, starting from the weights $(\alpha_i)$ and points $(v_i)$ obtained from the conditional gradient algorithm.

\paragraph{Adding a new basis function.}
The conditional gradient algorithm presented above relies on solving at each iteration the ``Frank-Wolfe step'':
$$ \max_{ \gamma(f) \leqslant \delta}  \ \langle f, g \rangle_{ L_2(d\rho)}.$$
for $g = - J'(f_t) \in  L_2(d\rho)$. For the norm $\gamma_1$ defined through an $L_1$-norm, we have for
$f = \int_{\V} \varphi_v d \mu(v)$ such that $\gamma_1(f) = |\mu|(\V)$:
\BEAS
\langle f, g \rangle_{ L_2(d\rho)}
& = & \int_\X f(x) g(x) d\rho(x) 
= \int_\X \bigg( \int_\V \varphi_v(x)  d \mu(v) \bigg) g(x)  d\rho(x) \\
&
= &  \int_\V \bigg( \int_\X \varphi_v(x) g(x)  d\rho(x) \bigg) d  \mu(v) \\
& \leqslant & \gamma_1(f) \cdot \max_{v \in \V} \bigg| \int_\X \varphi_v(x) g(x)  d\rho(x) \bigg|,
\EEAS
with equality if and only if   $\mu = \mu_+ - \mu_-$ with $\mu_+$ and $\mu_-$ two non-negative measures, with $\mu_+$ (resp.~$\mu_-$) supported in the set of maximizers $v$ of 
$ \big| \int_\X \varphi_v(x) g(x)  d\rho(x) \big|$ where the value is positive (resp. negative).

This implies that:
\BEQ
\label{eq:incpop}
\max_{ \gamma_1(f) \leqslant \delta}  \ \langle f, g \rangle_{ L_2(d\rho)}
= \delta \max_{v \in \V} \bigg| \int_\X \varphi_v(x) g(x)  d\rho(x) \bigg|,
\EEQ
with the maximizers $f$ of the first optimization problem above (left-hand side) obtained as $\delta$ times convex combinations of $\varphi_v$ and $-\varphi_v$ for maximizers  $v$ of the second problem (right-hand side). 

A common difficulty in practice is the hardness of the Frank-Wolfe step, that is, the optimization problem above over $\V$ may be difficult to solve. See \mysec{FW1}, \ref{sec:FW2} and \ref{sec:FW3} for  neural networks, where this optimization is usually difficult.

\paragraph{Finitely many observations.}
When $\X$ is finite (or when using the result from \mysec{rep}), the Frank-Wolfe step in \eq{incpop} becomes equivalent to,
for some  vector $g \in \rb^n$: 
\BEQ
\label{eq:incr}
  \sup_{ \gamma_1(f) \leqslant \delta }   \frac{1}{n} \sum_{i=1}^n g_i f(x_i)   
\ = \  \delta \max_{v \in \mathcal{V} }  \bigg| \frac{1}{n} \sum_{i=1}^n g_i \varphi_v(x_i) \bigg| ,
\EEQ
where the set of solutions of the first problem is in the convex hull of the solutions of the second problem.

\paragraph{Non-smooth loss functions.}
In this paper, in our theoretical results, we consider non-smooth loss functions for which conditional gradient algorithms do not converge in general. One possibility is to smooth the loss function, as done by~\citet{nesterov2005smooth}: an approximation error of $\varepsilon$ may be obtained with a smoothness constant proportional to $1/\varepsilon$. By choosing $\varepsilon$ as $1/\sqrt{t}$, we obtain a convergence rate of $O(1/\sqrt{t})$ after $t$ iterations. See also~\citet{lan}.

\paragraph{Approximate oracles.}
The conditional gradient algorithm may deal with approximate oracles; however, what we need in this paper  is not the additive errors situations considered by~\citet{jaggi}, but multiplicative ones on the computation of the dual norm (similar to ones derived by~\citet{siammatrix} for the regularized problem).

Indeed, in our context, we minimize a function $J(f)$ on $f \in L_2(d\rho)$ over a norm ball $\{ \gamma_1(f) \leqslant \delta\}$. A multiplicative approximate oracle outputs for any $g \in L_2(d\rho)$, a vector $\hat{f} \in L_2(d\rho)$ such that $\gamma_1(\hat{f})=1$, and
$$
\langle \hat{f}, g \rangle \leqslant \max_{\gamma_1(f) \leqslant 1} \langle f, g \rangle \leqslant \kappa \, \langle \hat{f}, g \rangle,
$$
for a fixed $\kappa \geqslant 1$. In Appendix~\ref{app:cg}, we propose a modification of the conditional gradient algorithm that converges to a certain $h\in L_2(d\rho)$ such that $\gamma_1(h) \leqslant \delta$ and for which
$\inf_{\gamma_1(f) \leqslant  \delta} J(f) \leqslant J(h) \leqslant \inf_{\gamma_1(f) \leqslant   \delta / \kappa} J(f)$.

Such approximate oracles are not available in general, because they require uniform bounds over all possible values of $g \in L_2(d\rho)$. In \mysec{oracle}, we show that a weaker form of oracle is sufficient to preserve our generalization bounds from \mysec{bounds}.

\paragraph{Approximation of any function by a finite number of basis functions.}
The Frank-Wolfe algorithm may be applied in the function space $\F_1$ with $J(f) = \frac{1}{2} \E \big[  (f(x) - g(x))^2 \big] $, we get a function~$f_t$, supported by $t$ basis functions such that 
$\displaystyle   \E \big[ (f_t(x) - g(x))^2\big]  = O( \gamma(g)^2  /t) $.
Hence, any function in $\F_1$ may be approximated with averaged error $\varepsilon$ with $t = O \big( \big[ \gamma(g)   / \varepsilon \big]^2 \big)$ units.
Note that the conditional gradient algorithm is one among many ways to obtain such approximation with $\varepsilon^{-2}$ units~\citep{barron1993universal,kurkova,mhaskar2004tractability}. See \mysec{neuron} for a (slightly) better dependence on $\varepsilon$ for convex neural networks.

\section{Neural networks with non-decreasing positively homogeneous activation functions}
\label{sec:neural}

In this paper, we focus on a specific family of basis functions, that is, of the form
$$ x \mapsto \sigma( w^\top x +   b ),$$
for specific activation functions $\sigma$. We assume that $\sigma $ is non-decreasing and positively homogeneous of some integer degree, i.e., it is equal to $\sigma(u) = (u)_+^\alpha$, for some $\alpha \in \{0,1,\dots\}$. We focus on these functions for several reasons:

\vspace*{-.125cm}

\BIT
\item[--] Since they are not polynomials, linear combinations of these functions can approximate any measurable function~\citep{leshno1993multilayer}.
\item[--] By homogeneity, they are invariant by a change of scale of the data; indeed, if all observations~$x$  are multiplied by a constant, we may simply change the measure $\mu$ defining the expansion of $f$ by the appropriate constant to obtain exactly  the same function. This allows us to study functions defined on the unit-sphere.

\item[--] The special case $\alpha=1$, often referred to as the \emph{rectified linear unit}, has seen considerable recent empirical success~\citep{nair2010rectified,krizhevsky2012imagenet}, while the case $\alpha=0$ (hard thresholds) has some historical importance~\citep{rosenblatt1958perceptron}.

\EIT
The goal of this section is to specialize the results from \mysec{general} to this particular case and show that the ``Frank-Wolfe'' steps have   simple geometric interpretations.

We first show that the positive homogeneity of the activation functions allows to transfer the problem to a unit sphere.

\paragraph{Boundedness assumptions.} For the theoretical analysis, we assume that our data inputs $x \in \rb^d$ are almost surely bounded by $R$ in $\ell_q$-norm, for some $q \in [2,\infty]$ (typically $q=2$ and $q=\infty$). We then build the augmented variable
$z \in \rb^{d+1}$ as $z = ( x^\top, R)^\top \in \rb^{d+1}$ by appending the constant~$R$ to $x \in \rb^d$. We therefore have $\| z\|_q \leqslant \sqrt{2} R$. By defining the vector $v=(w^\top,b/R)^\top \in \rb^{d+1}$, we have:
$$\varphi_v(x) = \sigma( w^\top x +   b ) = \sigma( v^\top z ) = (v^\top z)_+^\alpha,$$
which now becomes a function of $z \in \rb^{d+1}$.

Without loss of generality (and by homogeneity of $\sigma$), we may assume that the $\ell_p$-norm of each vector $v$ is equal to $1/R$, that is $\V$ will be the $(1/R)$-sphere for the $\ell_p$-norm, where $1/p+1/q=1$ (and thus $p\in [1,2]$, with corresponding typical values $p=2$ and $p=1$).

This implies by H\"older's inequality that $\varphi_v(x)^2 \leqslant 2^\alpha  $.  Moreover this leads to functions in $\F_1$ that are bounded everywhere, that is, $ \forall f \in \F_1$, $ f(x)^2 \leqslant  2^\alpha   \gamma_1(f)^2$.
Note that the functions in $\F_1$ are also Lipschitz-continuous for $\alpha \geqslant 1$. 

Since all $\ell_p$-norms (for $p \in [1,2]$) are equivalent to each other with constants of at most $\sqrt{d}$ with respect to the $\ell_2$-norm, all the spaces $\F_1$ defined above are equal, but the norms $\gamma_1$ are of course different and they differ by a constant of at most $d^{\alpha/2}$---this can be seen by computing the dual norms  like in \eq{incr} or \eq{incpop}.

\paragraph{Homogeneous reformulation.}
In our study of approximation properties, it will be useful to consider the   the space of function $\mathcal{G}_1$ defined for $z$ in the unit sphere $\Sb^d \subset \rb^{d+1}$ of the Euclidean norm, such that $ g(z) = \int_{\Sb^d} \sigma(v^\top z ) d \mu(v)$,  
with the norm $\gamma_1(g)$ defined as the infimum of $|\mu|(\Sb^d)$ over all decompositions of $g$. Note the slight overloading of notations for $\gamma_1$ (for norms in $\G_1$ and $\F_1$) which should not cause any confusion.

 In order to prove the approximation properties (with unspecified constants depending only on $d$), we may assume that $p=2$, since the norms $\| \cdot\|_p$ for $p \in [1,\infty]$ are equivalent to $\| \cdot\|_2$ with a constant that grows at most as $d^{\alpha/2}$ with respect to the $\ell_2$-norm. We thus focus on the $\ell_2$-norm in all proofs in \mysec{approx}.

We may go from $\mathcal{G}_1$ (a space of real-valued functions defined on the unit $\ell_2$-sphere in $d+1$ dimensions) to the space $\mathcal{F}_1$ (a space of real-valued functions defined on the ball of radius $R$ for the $\ell_2$-norm) as follows (this corresponds to sending a ball in $\rb^d$ into a spherical cap in dimension $d+1$, as illustrated in \myfig{cap}).

\BIT
\item[--] Given $g \in \mathcal{G}_1$, we define $ f \in \mathcal{F}_1$, with $f(x) = \Big( \frac{\| x\|_2^2}{R^2} + 1 \Big)^{\alpha/2} g\bigg( \displaystyle
\frac{1}{ \sqrt{ \| x\|_2^2 + R^2  }} {x \choose R} \bigg)$. If $g$ may be represented as 
$\int_{\Sb^d} \sigma(v^\top z ) d \mu(v)$, then the function $f$ that we have defined may be represented
as 
\BEAS
f(x) & = &  \Big( \frac{\| x\|_2^2}{R^2} + 1 \Big)^{\alpha/2} \int_{\Sb^d}   \bigg( v^\top \displaystyle
\frac{1}{ \sqrt{ \| x\|_2^2 + R^2  }} {x \choose R} \bigg)_+^\alpha d \mu(v) \\
& = & \int_{\Sb^d}   \bigg( v^\top \displaystyle
  {x/R \choose 1} \bigg)_+^\alpha d \mu(v) =
 \int_{\Sb^d} \sigma(w^\top x + b) d \mu(Rw,b),
\EEAS that is $\gamma_1(f) \leqslant \gamma_1(g)$, because we have assumed that $(w^\top,b/R)^\top $ is on the $(1/R)$-sphere.

\item[--] Conversely, given $f \in \mathcal{F}_1$, for $z = (t^\top, a)^\top \in \Sb^{d}$, we define
$g(z) = g(t,a) =   f\big( \frac{Rt}{a} \big)  a^\alpha  $, which we define as such on the set of $z= (t^\top, a)^\top \in \rb^d \times \rb$ (of unit norm) such that $a \geqslant \frac{1}{\sqrt{2}}$. Since we always assume $\| x\|_2 \leqslant R$,  we have  $ \sqrt{ \| x\|_2^2 + R^2  } \leqslant \sqrt{2}R$, and the value of $g(z,a)$ for $a \geqslant \frac{1}{\sqrt{2}}$ is enough to recover $f$ from the formula above.

\begin{figure}
\begin{center}
\includegraphics[scale=1]{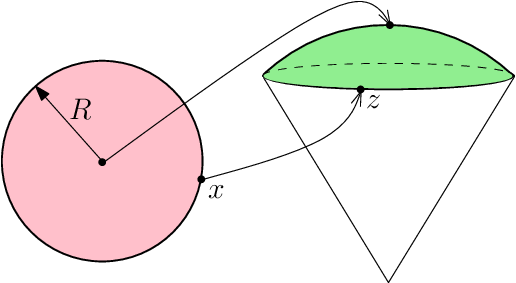}
\end{center}

\vspace*{-.5cm}

\caption{ Sending a ball to a spherical cap.
\label{fig:cap} }
\end{figure}

 On that portion $\{ a \geqslant 1/\sqrt{2} \}$ of the sphere $\Sb^d$, this function exactly inherits the differentiability properties of $f$. That is, (a) if $f$ is bounded by $1$ and $f$ is $(1/R)$-Lipschitz-continuous, then $g$ is Lipschitz-continuous with a constant that only depends on $d$ and $\alpha$ and (b), if all derivatives of order less than $k$ are bounded by $R^{-k}$, then all derivatives of the same order of $g$ are bounded by a constant that only depends on $d$ and $\alpha$. Precise notions of differentiability may be defined on the sphere, using the manifold structure~\cite[see, e.g.,][]{absil2009optimization} or through polar coordinates~\cite[see, e.g.,][Chapter 3]{atkinson2012spherical}. See these references for more details.

 The only remaining important aspect is to define $g$ on the entire sphere, so that (a) its regularity constants are controlled by a constant times the ones on the portion of the sphere where it is already defined, (b) $g$ is either even or odd (this will be important in \mysec{approx}). Ensuring that the regularity conditions can be met is classical when extending to the full sphere~\citep[see, e.g.,][]{whitney1934analytic}. Ensuring  that the function may be chosen as odd or even may be obtained by multiplying the function $g$ by an infinitely differentiable function which is equal to one for $a \geqslant 1/\sqrt{2}$ and zero for $a \leqslant 0$, and extending by $-g$ or $g$ on the hemi-sphere $a<0$.

\EIT

In summary, we may consider in \mysec{approx} functions defined on the sphere, which are much easier to analyze. In the rest of the section, we specialize some of the general concepts reviewed in \mysec{general} to our neural network setting with specific activation functions, namely, in terms of corresponding kernel functions and geometric reformulations of the Frank-Wolfe steps.

\subsection{Corresponding positive-definite kernels}
\label{sec:kersigma}
In this section, we consider the $\ell_2$-norm on the input weight vectors $w$ (that is $p=2$). We may compute for
$x,x' \in \rb^{d}$ the kernels defined in \mysec{RKHS}:
$$
k_\alpha(x,x') = \E \big[  (w^\top x + b)_+^\alpha  (w^\top x' + b)_+^\alpha \big],
$$
for $(Rw,b)$ distributed uniformly on the unit $\ell_2$-sphere $\Sb^{d}$, 
and $x,x' \in \rb^{d+1}$. Given the angle $\varphi \in[0,\pi]$ defined through $\displaystyle \frac{x^\top x'}{R^2} + 1 =  ( \cos \varphi ) \sqrt{\frac{\|x\|_2^2 }{R^2} + 1 }  \sqrt{\frac{\|x'\|_2^2 }{R^2} + 1 }$, we
have explicit expressions~\citep{roux2007continuous,cho2009kernel}: 
\BEAS
k_0(z,z') & = & 
\frac{1}{2\pi} 
\big(
  \pi - \varphi 
\big)
\\
k_1(z,z') & = & 
\frac{ \sqrt{\frac{\|x\|_2^2 }{R^2} + 1 }  \sqrt{\frac{\|x'\|_2^2 }{R^2} + 1 } }{2(d+1)\pi} 
\big(
( \pi - \varphi) \cos \varphi + \sin \varphi
\big)
\\
k_2(z,z') & = & 
\frac{ \Big({\frac{\|x\|_2^2 }{R^2} + 1 }\Big)  \Big({\frac{\|x'\|_2^2 }{R^2} + 1 } \Big) }{ 2 \pi [  (d+1)^2 + 2(d+1) ]} 
\big(
3 \sin \varphi
  \cos \varphi + ( \pi - \varphi) ( 1 + 2 \cos^2 \varphi)
\big).
\EEAS
There are key differences and similarities between the RKHS $\F_2$ and our space of functions $\F_1$. The RKHS is smaller than $\F_1$ (i.e., the norm in the RKHS is larger than the norm in $\F_1$); this implies that approximation properties of the RKHS are transferred to $\F_1$. In fact, our proofs rely on this fact.

However, the  RKHS norm does not lead to any adaptivity, while the function space $\F_1$ does (see more details in \mysec{adaptivity}). This may come as a paradox: both the RKHS $\F_2$ and $\F_1$ have similar properties, but one is adaptive while the other one is not. A key intuitive difference is as follows: given a function
$f$ expressed as $ f(x) = \int_\V  \varphi_v(x) {p}(v) d \tau(v)$, then $\gamma_1(f) = \int_\V  |{p}(v)| d \tau(v)$, while the squared RKHS norm is  $\gamma_2(f)^2 = \int_\V  |{p}(v)|^2 d \tau(v)$. For the $L_1$-norm, the measure ${p}(v) d \tau(v)$ may tend to a singular distribution with a bounded norm, while this is not true for the $L_2$-norm. For example, the function $(w^\top x + b)_+^\alpha$ is in $\mathcal{F}_1$, while it is not in $\F_2$ in general.

\subsection{Incremental optimization problem for $\alpha=0$}
\label{sec:FW1}
We consider the problem in \eq{incr} for the special case $\alpha=0$. For $z_1,\dots,z_n \in \rb^{d+1}$ and a vector $y \in \rb^n$, the goal is to solve (as well as the corresponding problem with $y$ replaced by $-y$):
$$
\max_{ v \in \rb^{d+1} }  \sum_{i=1}^n y_i 1_{ v^\top z_i > 0 }  
=\max_{ v \in \rb^{d+1} }   \ \ \sum_{i \in I_+ } |y_i| 1_{ v^\top z_i > 0 }
- \sum_{i \in I_- }  |y_i| 1_{ v^\top z_i > 0 }  ,
$$
where $I_+ = \{i, y_i \geqslant 0\}$ and $I_- = \{i, y_i < 0\}$. As outlined by~\citet{bengio2006convex}, this is equivalent to finding an hyperplane parameterized by $v$ that minimizes a weighted mis-classification rate (when doing linear classification). Note that the norm of $v$ has no effect.

\paragraph{NP-hardness.}
This problem is NP-hard in general. Indeed, if we assume that all $y_i$ are equal to $-1$ or $1$ and with $\sum_{i=1}^n y_i = 0$, then we have a balanced binary classification problem (we need to assume $n$ even).
The quantity $\sum_{i=1}^n y_i 1_{ v^\top z_i > 0 }  $ is then $\frac{n}{2}(1-2 e)$ where $e$ is the corresponding classification error for a problem of classifying at positive (resp.~negative) the examples in $I_+$ (resp.~$I_-$) by thresholding the linear classifier $v^\top z$. \citet{guruswami2009hardness} showed that for all $(\varepsilon,\delta)$, it is NP-hard to  distinguish between instances (i.e., configurations of points $x_i$), where a halfspace with  classification error at most $\varepsilon$ exists, and instances where all half-spaces have an error of at least $1/2-\delta$. Thus, it is NP-hard to distinguish between instances where there exists $v \in \rb^{d+1}$ such that $\sum_{i=1}^n y_i 1_{ v^\top z_i > 0 }
\geqslant \frac{n}{2}(1-2 \varepsilon)$ and instances where for all $v \in \rb^{d+1}$, 
$\sum_{i=1}^n y_i 1_{ v^\top z_i > 0 }
\leqslant n \delta$. Thus, it is NP-hard to distinguish instances where $\max_{ v \in \rb^{d+1} }  \sum_{i=1}^n y_i 1_{ v^\top z_i > 0 }   \geqslant \frac{n}{2} ( 1- 2 \varepsilon)$ and ones where it is less than $\frac{n}{2}\delta$. Since this is valid for all $\delta$ and $\varepsilon$, this rules out a constant-factor approximation.

\paragraph{Convex relaxation.}
Given linear binary classification problems, there are several algorithms to approximately find a good half-space. These are based on using convex surrogates (such as the hinge loss or the logistic loss). Although some theoretical results do exist regarding the classification performance of estimators obtained from convex surrogates~\citep{bartlett2006convexity}, they do not apply in the context of linear classification.

\subsection{Incremental optimization problem for $\alpha=1$}
\label{sec:FW2}
We consider the problem in \eq{incr} for the special case $\alpha=1$. For $z_1,\dots,z_n \in \rb^{d+1} $ and a vector $y \in \rb^n$, the goal is to solve (as well as the corresponding problem with $y$ replaced by $-y$):
$$
\max_{ \| v\|_p \leqslant 1 }  \sum_{i=1}^n y_i  (v^\top z_i)_+
=\max_{ \| v\|_p \leqslant 1 }  \ \ \sum_{i \in I_+ } (v^\top  |y_i| z_i)_+
- \sum_{i \in I_- }  (v^\top  |y_i| z_i)_+,
$$
where $I_+ = \{i, y_i \geqslant 0\}$ and $I_- = \{i, y_i < 0\}$. We have, with $t_i = |y_i| z_i \in \rb^{d+1}$,
using convex duality:

\BEAS
 \max_{ \| v\|_p \leqslant 1 } \ \sum_{i=1}^n y_i (v^\top z_i)_+   
& = & \max_{\| v\|_p  \leqslant 1}  \ \sum_{i \in I_+ } (v^\top  t_i)_+
- \sum_{i \in I_- }  (v^\top  t_i)_+\\
   & = & 
   \max_{\| v\|_p \leqslant 1}  \ \sum_{i \in I_+ }    \max_{ b_i \in [0,1]} b_i v^\top t_i
   - \sum_{i \in I_- }   \max_{ b_i \in [0,1] } b_i v^\top t_i  \\
   & = & 
    \max_{b_+ \in [0,1]^{I_+}}  
    \max_{\| v\|_p \leqslant 1} \min_{b_- \in [0,1]^{I_-}}   v^\top \big[ T_+^\top b_+ -  T_-^\top b_- \big]
\\
   & = & 
    \max_{b_+ \in [0,1]^{I_+}}  \min_{b_- \in [0,1]^{I_-}} 
    \max_{\| v\|_p \leqslant 1}  v^\top \big[ T_+^\top b_+ -  T_-^\top b_- \big]
    \mbox{ by Fenchel duality,}\\
   & = & 
   \max_{b_+ \in [0,1]^{I_+}}  \min_{b_- \in [0,1]^{I_-}} 
   \big\|  T_+^\top b_+ -  T_-^\top b_- \big\|_q,
  \EEAS
  where $T_+ \in \rb^{n_+ \times d}$ has rows $t_i$, $i \in I_+$ and 
 $T_- \in \rb^{n_- \times d}$ has rows $t_i$, $i \in I_-$,
with $v \in \arg \max_{ \| v\|_p \leqslant 1} v^\top (  T_+^\top b_+ -  T_-^\top b_- )$.
The problem thus becomes 
$$
   \max_{b_+ \in [0,1]^{n_+}}  \min_{b_- \in [0,1]^{n_-}} 
  \big\| T_+^\top b_+ -  T_-^\top b_-\big\|_q.
   $$

For the problem of maximizing $\big|\sum_{i=1}^n y_i (v^\top z_i)_+ \big|$, then this corresponds to
$$
  \max\bigg\{ \max_{b_+ \in [0,1]^{n_+}}  \min_{b_- \in [0,1]^{n_-}} 
   \big \|  T_+^\top b_+ -  T_-^\top b_-\big\|_q
   ,
   \max_{b_- \in [0,1]^{n_-}}    \min_{b_+ \in [0,1]^{n_+}} 
   \big\| T_+^\top b_+ -  T_-^\top b_-\big\|_q
   \bigg\}.
$$

This is exactly the Hausdorff distance between the two convex sets $\{  
T_+^\top b_+, \ b_+ \in  [0,1]^{n_+}\}$ and  $\{  
T_-^\top b_-, \ b_- \in  [0,1]^{n_-}\}$  (referred to as zonotopes, see below).

Given the pair $(b_+,b_-)$ achieving the Hausdorff distance, then we may compute the optimal $v$ as 
$v =\arg \max_{ \| v\|_p \leqslant 1} v^\top \big(T_+^\top b_+ - T_-^\top b_-\big)$. 
Note this has not changed the problem at all, since it is equivalent. It is still  NP-hard in general~\citep{Konig14}. But we now have a geometric interpretation with potential approximation algorithms. See below and \mysec{approximation}.

 \begin{figure}
\begin{center}
 \includegraphics[scale=.85]{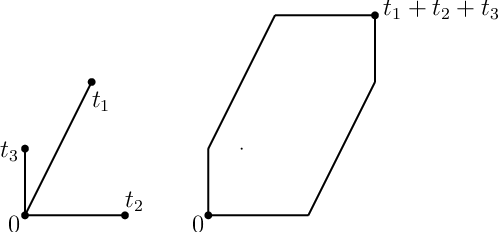}
 \includegraphics[scale=.85]{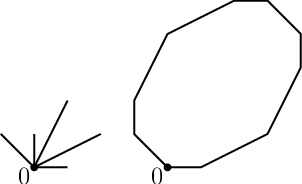}
 \end{center}
 \caption{Two zonotopes in two dimensions: (left)  vectors, and (right) their Minkowski sum (represented as a polygone).}
 \label{fig:zonotope}
 \end{figure}

 \paragraph{Zonotopes.} A  \emph{zonotope} $A$ is the Minkowski sum of a finite number of segments from the origin, that is, of the form 
 $$A =  [0 , t_1] + \cdots + [ 0, t_r ] = \Big\{ \sum_{i=1}^r  b_i t_i, \ b \in [0,1]^r \Big\},$$ for some vectors   $t_i$, $i=1,\dots,r$~\citep{bolker1969class}. See an illustration in \myfig{zonotope}.
 They appear in several areas of computer science~\citep{edelsbrunner1987algorithms,guibas2003zonotopes} and mathematics~\citep{bolker1969class,bourgain1989approximation}. In machine learning, they appear naturally as the affine projection of a hypercube; in particular,  when using a higher-dimensional distributed representation of points in~$\rb^d$ with elements in $[0,1]^r$, where $r $ is larger than $d$~\citep[see, e.g.,][]{hinton1997generative}, the underlying polytope that is modelled in $\rb^d$ happens to be a zonotope.

 In our context, the  two convex sets $\{  
T_+^\top b_+, \ b_+ \in  [0,1]^{n_+}\}$ and  $\{  
T_-^\top b_-, \ b_- \in  [0,1]^{n_-}\}$ defined above are thus zonotopes. See an illustration of the Hausdorff distance computation in \myfig{haus} (middle plot), which is the core computational problem for $\alpha=1$.

\begin{figure}
\begin{center}
 \includegraphics[scale=.85]{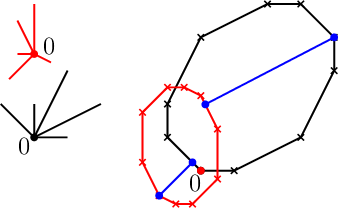}
\hspace{1.5cm}
 \includegraphics[scale=.85]{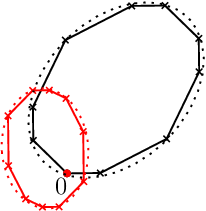}
 \end{center}
 \caption{Left: two zonotopes (with their generating segments) and the segments achieving the two sides of the Haussdorf distance. Right: approximation by ellipsoids.}
 \label{fig:haus}
\end{figure}

\paragraph{Approximation by ellipsoids.}

Centrally symmetric convex polytopes (w.l.o.g.~centered around zero) may be approximated by ellipsoids. In our set-up, we could use the minimum volume enclosing ellipsoid~\citep[see, e.g.][]{alexander2002course}, which can be computed   exactly when the polytope is given through its vertices, or up to a constant factor when the polytope is such that quadratic functions may be optimized with a constant factor approximation. For zonotopes, the standard semi-definite relaxation of~\citet{nesterov1998semidefinite} leads to such constant-factor approximations, and thus the minimum volume inscribed ellipsoid may be computed up to a constant. Given standard results~\citep[see, e.g.][]{alexander2002course}, a $(1/\sqrt{d})$-scaled version of the ellipsoid is inscribed in this polytope, and thus the ellipsoid is a provably good approximation of the zonotope with a factor scaling as $\sqrt{d}$. However, the approximation ratio is not good enough to get any relevant bound for our purpose (see \mysec{sufficient}), as for computing the Haussdorff distance, we care about potentially vanishing differences that are swamped by constant factor approximations.
 
Nevertheless, the ellipsoid approximation may prove useful in practice, in particular because the $\ell_2$-Haussdorff distance between two ellipsoids may be computed in polynomial time (see Appendix~\ref{app:ellipsoid}).

\paragraph{NP-hardness.} Given the reduction of the case $\alpha=1$ (rectified linear units) to $\alpha=0$ (exact thresholds)~\citep{livni2014computational}, the incremental problem is also NP-hard, so as obtaining a constant-factor approximation. However, this does not rule out convex relaxations with non-constant approximation ratios (see \mysec{gauge} for more details).

\subsection{Incremental optimization problem for $\alpha \geqslant 2$}
\label{sec:FW3}

We consider the problem in \eq{incr} for the remaining cases $\alpha \geqslant 2$. For $z_1,\dots,z_n \in \rb^{d+1}$ and a vector $y \in \rb^n$, the goal is to solve (as well as the corresponding problem with $y$ replaced by $-y$):
$$
\max_{ \| v\|_p \leqslant 1 }  \frac{1}{\alpha} \sum_{i=1}^n y_i  (v^\top z_i)_+^\alpha
=\max_{ \| v\|_p \leqslant 1 }  \sum_{i \in I_+ }   \frac{1}{\alpha}   (v^\top   |y_i|^{1/\alpha} z_i)_+^\alpha
- \sum_{i \in I_- }  \frac{1}{\alpha}  (v^\top |y_i|^{1/\alpha} z_i)_+^\alpha,
$$
where $I_+ = \{i, y_i \geqslant 0\}$ and $I_- = \{i, y_i < 0\}$. We have, with $t_i =  |y_i|^{1/\alpha} z_i \in \rb^{d+1}$, and $\beta \in (1,2]$ defined by $1/\beta + 1/\alpha = 1$ (we use the fact that the function $u \mapsto u^\alpha/\alpha$ and $v \mapsto v^\beta /  \beta $ are Fenchel-dual to each other):

\BEA
\nonumber \max_{ \| v\|_p \leqslant 1 } \ \frac{1}{\alpha} \sum_{i=1}^n y_i (v^\top z_i)_+  ^\alpha 
& = & \max_{\| v\|_p  \leqslant 1}  \ \sum_{i \in I_+ } \frac{1}{\alpha}  (v^\top  t_i)_+^\alpha 
- \sum_{i \in I_- }  \frac{1}{\alpha} (v^\top  t_i)_+^\alpha \\
 \nonumber  & = & 
   \max_{\| v\|_p \leqslant 1}  \ \sum_{i \in I_+ }    \max_{b_i \geqslant 0 } \Big\{ b_i v_i^\top t_i - \frac{1}{\beta} b_i^\beta \Big\}
   - \sum_{i \in I_- }   \max_{ b_i \geqslant 0 } \Big\{ b_i v^\top t_i - \frac{1}{\beta} b_i^\beta \Big\}\\
  \nonumber & = & 
    \max_{b_+ \in \rb_+^{I_+}}  \min_{b_- \in \rb_+^{I_-}} 
    \max_{\| v\|_p \leqslant 1}  v^\top \big[ T_+^\top b_+ -  T_-^\top b_- \big]
    - \frac{1}{\beta} \| b_+\|_\beta^\beta
   + \frac{1}{\beta} \| b_-\|_\beta^\beta \\
   & & \hspace{7cm}
\nonumber    \mbox{ by Fenchel duality,}\\
\label{eq:alpha}   & = & 
   \max_{b_+ \in [0,1]^{I_+}}  \min_{b_- \in [0,1]^{I_-}} 
   \big\|  T_+^\top b_+ -  T_-^\top b_- \big\|_q - \frac{1}{\beta} \| b_+\|_\beta^\beta
   + \frac{1}{\beta} \| b_-\|_\beta^\beta,
  \EEA
  where $T_+ \in \rb^{n_+ \times d}$ has rows $t_i$, $i \in I_+$ and 
 $T_- \in \rb^{n_- \times d}$ has rows $t_i$, $i \in I_-$,
with $v   \in \arg \max_{ \| v\|_p \leqslant 1} \big( T_+^\top b_+ -  T_-^\top b_-\big)^\top v$.
 Contrary to the case $\alpha=1$, we do not obtain exactly a formulation as a Hausdorff distance. However, if we consider the convex sets $K^+_\lambda = \{ T_+^\top b_+, \ b_+ \geqslant 0, \ \| b_+ \|_\beta \leqslant \lambda\}$ and $K^-_\mu = \{ T_-^\top b_-, \ b_- \geqslant 0, \ \| b_- \|_\beta \leqslant \mu\}$, then, a solution of \eq{alpha} may be obtained from Hausdorff distance computations between $K^+_\lambda$ and $K^-_\mu$, for certain $\lambda$ and $\mu$.

Note that, while for $\alpha=1$ we can use the identity $2 u_+ = u + |u|$ to replace the rectified linear unit by the absolute value and obtain the same function space, this is not possible for $\alpha=2$, as $(u_+)^2$ and $u^2$ do not differ by a linear function. This implies that the results from~\cite{livni2014computational}, which state that for the quadratic activation function, the incremental problems is equivalent to an eigendecomposition (and hence solvable in polynomial time), do not apply.

 \section{Approximation properties}
 \label{sec:approx}
 In this section, we consider the approximation properties of the set $\F_1$ of functions defined on $\rb^d$. 
 As mentioned earlier, the norm used to penalize input weights $w$ or $v$ is irrelevant for approximation properties as all norms are equivalent. Therefore, we focus on the case $q=p=2$ and $\ell_2$-norm constraints.
 
 Because we consider homogeneous activation functions, we start by studying the set $\G_1$ of functions defined on the unit $\ell_2$-sphere $\Sb^d \subset \rb^{d+1}$. We denote by $\tau_d$ the uniform probability measure on $\Sb^d$. The set $\G_1$ is defined as the set of functions on the sphere
  such that $ g(z) = \int_{\Sb^d} \sigma(v^\top z ) p(z) d\tau_d(z)$, with the norm  $\gamma_1(g)$ equal to the smallest possible value of $\int_{\Sb^d}  |p(z)| d\tau_d(z)$. We may also define the corresponding squared RKHS norm by the smallest possible value of $\int_{\Sb^d}  |p(z)|^2 d\tau_d(z)$, with the corresponding RKHS $\G_2$.
  
  In this section, we first consider approximation properties of functions in $\G_1$ by a finite number of neurons (only for $\alpha=1$). We then study approximation properties of functions on the sphere by functions in $\G_1$. It turns out that all our results are based on the approximation properties of the corresponding RKHS $\G_2$: we give sufficient conditions for being in $\G_2$, and then approximation bounds for functions which are not in $\G_2$. Finally we transfer these to the spaces $\F_1$ and $\F_2$, and consider in particular functions which only depend on projections on a low-dimensional subspace, for which the properties of $\G_1$ and $\G_2$ (and of $\F_1$ and $\F_2$) differ. This property is key to obtaining generalization bounds that show adaptivity to linear structures in the prediction functions (as done in \mysec{bounds}).
 
 Approximation properties of neural networks with finitely many neurons have been studied extensively~\citep[see, e.g.,][]{petrushev1998approximation,pinkus1999approximation,makovoz1998uniform,burger2001error}. In \mysec{related}, we relate our new results  to existing work from the literature on approximation theory, by showing that our results provide an explicit control  of the various weight vectors which are needed for bounding the estimation error in \mysec{bounds}.

 \subsection{Approximation by a finite number of basis functions}
 \label{sec:neuron}
 
A key quantity that drives the approximability by a finite number of neurons is the variation norm $\gamma_1(g)$. As shown in \mysec{condgrad}, any function $g$ such that $\gamma_1(g)$ is finite, may be approximated in $L_2(\Sb^d)$-norm with error $\varepsilon$ with $n = O( \gamma_1(g) ^2\varepsilon^{-2})$ units.
For $\alpha=1$ (rectified linear units), we may improve the dependence in $\varepsilon$, through the link with zonoids and zonotopes, as we now present.

If we decompose the signed measure $\mu$ as $\mu = \mu_+ - \mu_-$ where $\mu_+$ and $\mu_-$ are positive measures, then, for $g\in \G_1$, we have
 $g(z)  = \int_{\Sb^d} (v^\top z)_+ d\mu_+(v) - \int_{\Sb^d} (v^\top z )_+ d\mu_-(v) =  {g}_+(z) -  {g}_-(z)$, which is a decomposition of $g$ as a difference of positively homogenous convex functions. 
 
 Positively homogenous convex functions $h$ may be written as the \emph{support function} of a compact convex set $K$~\citep{rockafellar97}, that is, $h(z) = \max_{ y \in K} y^\top z$, and the set $K$ characterizes the function $h$. The functions $g_+$ and $g_-$ defined above are not \emph{any} convex positively homogeneous functions, as we now describe.
 
 If the measure $\mu_+$ is supported by finitely many points, that is, $\mu_+(v) = \sum_{i=1}^r \eta_i \delta(v-v_i)$ with $\eta \geqslant 0$, then $g_+(z) = \sum_{i=1}^t \eta_i ( v_i^\top z)_+
 =  \sum_{i=1}^t  ( \eta_i v_i^\top z)_+ = \sum_{i=1}^t  ( t_i^\top z)_+$ for $t_i = \eta_i v_i$. Thus the corresponding set $K_+$ is the \emph{zonotope}
 $[0 , t_1] + \cdots + [ 0, t_r ] = \big\{ \sum_{i=1}^r  b_i t_i, \ b \in [0,1]^r \big\}$ already defined in~\mysec{FW2}. Thus the functions $g_+ \in \G_1$ and $g_- \in \G_1$ for finitely supported measures $\mu$ are support functions of zonotopes.
 
 When the measure $\mu$ is not constrained to have finite support, then the sets $K_+$ and $K_-$ are limits of zonotopes, and thus, by definition, \emph{zonoids}~\citep{bolker1969class}, and thus functions in $\G_1$ are differences of support functions of zonoids.
 Zonoids are a well-studied set of convex bodies. They are centrally symmetric, and in two dimensions, all centrally symmetric compact convexs sets are (up to translation) zonoids, which is not true in higher dimensions~\citep{bolker1969class}. Moreover, the problem of approximating a zonoid by a zonotope with a small number of segments~\citep{bourgain1989approximation,matouvsek1996improved} is essentially equivalent to the approximation of a function $g$ by finitely many neurons. The number of neurons directly depends on the norm $\gamma_1$, as we now show.

 \begin{proposition}[Number of units - $\alpha=1$]
 \label{prop:neurons}
Let  $\varepsilon \in (0,1/2)$. For any  function $g $ in $\G_1$, there exists a measure ${\mu}$ supported on at most $r$ points in $\V$, so that for all $z \in \Sb^{d}$.
$\big| g(z)  - \int_{\Sb^d} (v^\top z)_+ d{\mu}(v) \big| \leqslant \varepsilon \gamma_1(g) $, with 
$ r \leqslant  C(d)   \varepsilon  ^{-2d/ ( d +3 )} $, for some  constant $C(d)$ that depends only on $d$. \end{proposition}
\begin{proof}
Without loss of generality, we assume $\gamma(g)=1$. 
It is shown by~\citet{matouvsek1996improved} that for any probability measure $\mu$ (positive and with finite mass) on the sphere $\Sb^d$,   there exists a set of $r$ points $v_1,\dots,v_r$, so that
for all $z \in \Sb^d$,
\BEQ
\label{eq:matousek}
\bigg| \int_{\Sb^d} |v^\top z| d \mu (v) - \frac{1}{r} \sum_{i=1}^r |v_i^\top z| \bigg| \leqslant \varepsilon,
\EEQ
with $r \leqslant C(d) \varepsilon^{-2 + 6 / ( d + 3 )} = C(d)  \varepsilon  ^{-2d/ ( d +3 )}$, for some  constant $C(d)$ that depends only on $d$. We may then simply write 
$$\!g(z)\! =\! \int_{\Sb^d} \! (v^\top z)_+ d \mu(v) \!=\! \frac{1}{2}\int_{\Sb^d}\! (v^\top z) d \mu(v)
+  \frac{\mu_+( {\Sb^d}) }{2 }\! \int_{\Sb^d} \!|v^\top z| \frac{d \mu_+(v)}{\mu_+( {\Sb^d})} -    \frac{\mu_-( {\Sb^d})}{2}\!\int_{\Sb^d} \! |v^\top z|\frac{ d \mu_-(v) }{\mu_-( {\Sb^d})},$$ 
and approximate the last two terms with error $\varepsilon \mu_\pm(\Sb^d) $ with $r$ terms, leading to an approximation of $\varepsilon \mu_+(\Sb^d) +  \varepsilon\mu_-(\Sb^d)  = \varepsilon \gamma_1(g) = \varepsilon$, with a remainder that is a linear function $q^\top z$ of $z$, with $\| q\|_2 \leqslant 1$. We may then simply add two extra units with vectors $q/\|q\|_2$ and weights $-\|q\|_2$ and $\|q\|_2$. We thus obtain, with $2r+2$ units,  the desired approximation result.

Note that \citet[Theorem 6.5]{bourgain1989approximation} showed that  the scaling in $\varepsilon$  in \eq{matousek} is not improvable, if the measure is allowed to have non equal weights on all points and the proof relies on the non-approximability of the Euclidean ball by centered zonotopes. This results does not apply here, because we may have different weights $\mu_-( {\Sb^d})$ and $\mu_+( {\Sb^d})$.
\end{proof}

Note that the proposition above is slightly improved in terms of the scaling of the number of neurons with respect to the approximation error $\varepsilon$ (improved exponent), compared to conditional gradient bounds~\citep{barron1993universal,kurkova}.
Indeed, the simple use of conditional gradient leads to $r \leqslant \varepsilon^{-2} \gamma_1(g)^2$, with a better constant (independent of $d$) but a worse scaling in $\varepsilon$---also with a result in $L_2(\Sb^d)$-norm and not uniformly on the ball $\{ \| x\|_q \leqslant R\}$. Note also that the conditional gradient algorithm gives a constructive way of building the measure. 
Moreover,  the proposition above  is related to the result from~\citet[Theorem 2]{makovoz1998uniform}, which applies for $\alpha=0$ but with a number of neurons growing as $\varepsilon  ^{-2d/ ( d +1 )}$, or to the one of~\citet[Example 3.1]{burger2001error}, which applies to a piecewise affine sigmoidal function but with a number of neurons growing as $\varepsilon ^{-2(d+1)/ ( d +3 )}$ (both slightly worse than ours).  

Finally, the number of neurons needed to express a function with a bound on the $\gamma_2$-norm can be estimated from general results on approximating reproducing kernel Hilbert space described in \mysec{rkhs}, whose kernel can be expressed as an expectation. Indeed, \citet{kernelexp} shows that with $k$ neurons, one can approximate a function in $\F_2$ with unit $\gamma_2$-norm with an error measured in $L_2$ of $\varepsilon = k^{-(d+3)/(2d)}$. When inverting the relationship between $k$ and $\varepsilon$, we get a number of neurons scaling as $  \varepsilon  ^{-2d/ ( d +3 )}$, which is the same as in Prop.~\ref{prop:neurons} but with an error in $L_2$-norm instead of $L^\infty$-norm.

 \subsection{Sufficient conditions for finite variation}
  In this section and the next one, we study more precisely the RKHS $\G_2$ (and thus obtain similar results for $\G_1 \supset \G_2$). The kernel $k(x,y) = \int_{\Sb^d} ( v^\top x)_+ ( v^\top y)_+ d \tau_d(v) $ defined on the sphere $\Sb^d$ belongs to the family of dot-product kernels~\citep{smola2001regularization} that only depends on the dot-product $x^\top y$, although in our situation, the function is not particularly simple (see formulas in \mysec{kersigma}). The analysis of these kernels is similar to one of translation-invariant kernels; for $d=1$, i.e., on the $2$-dimensional sphere, it is done through Fourier series; while for $d>1$, \emph{spherical harmonics} have to be used as the expansion of functions in series of spherical harmonics make the computation of the RKHS norm explicit (see a review of spherical harmonics in Appendix~\ref{app:review} with several references therein). Since the calculus is tedious, all proofs are put in appendices, and we only present here the main results. In this section, we provide simple sufficient conditions for belonging to $\G_2$ (and hence $\G_1$) based on the existence and boundedness of derivatives, while in the next section, we show how any Lipschitz-function may be approximated by functions in $\G_2$ (and hence $\G_1$) with precise control of the norm of the approximating functions.

 The derivatives of functions defined on $\Sb^d$ may be defined in several ways, using the manifold structure~\cite[see, e.g.,][]{absil2009optimization} or through polar coordinates~\cite[see, e.g.,][Chapter 3]{atkinson2012spherical}. For $d=1$, the one-dimensional sphere $\Sb^1 \subset\rb^2$  may be parameterized by a single angle and thus the notion of derivatives and the proof of the following result is simpler  and based on Fourier series (see Appendix~\ref{app:proof11}). For the general proof based on spherical harmonics, see Appendix~\ref{app:harmonics}.

 \begin{proposition}[Finite variation on the sphere]
\label{prop:finites-sphere}
Assume that $g: \Sb^d \to \rb$ is such that all $i$-th order derivatives exist and are upper-bounded in absolute value
 by $ \eta  $ for $i \in \{0,\dots,s\}$,
where $s$ is an integer such that $s \geqslant (d - 1)/2 + \alpha + 1$. Assume $g$ is even if $\alpha$ is odd (and vice-versa); 
 then $g \in \G_2$ and $\gamma_2(g) \leqslant C(d,\alpha) \eta$, for  a constant $C(d,\alpha)$ that depends only on $d$ and $\alpha$.
\end{proposition}

We can make the following observations:

\vspace*{-.125cm}

\BIT
\item[--] \emph{Tightness of conditions}: as shown in Appendix~\ref{app:tightness}, there are functions $g $, which have bounded first $s$ derivatives and do not belong  to $\G_2$ while $s \leqslant \frac{d}{2} + \alpha$ (at least when $s-\alpha$ is even). Therefore,  when $s-\alpha$ is even,  the  scaling in $ {(d-1)}/{2}+ \alpha$ is optimal.

\item[--] \emph{Dependence on $\alpha$}: for any $d$, the higher the $\alpha$, the stricter the sufficient condition. Given that the estimation error grows slowly with $\alpha$ (see \mysec{rademacher}), low values of $\alpha$ would be preferred in practice.

\item[--] \emph{Dependence on $d$}: a key feature of the sufficient condition is the dependence on $d$, that is, as $d$ increases the number of derivatives has to increase  in $d/2$---like for Sobolev spaces in dimension $d$~\citep{adams2003sobolev}.
This is another instantiation of the curse of dimensionality: only very smooth functions in high dimensions are allowed.

\item[--] \emph{Special case $d=1$, $\alpha=0$}: differentiable functions on the sphere in $\rb^2$, with bounded derivatives, belong to $\G_2$, and thus all Lipschitz-continuous functions, because Lipschitz-continuous functions are almost everywhere differentiable with bounded derivative \citep{adams2003sobolev}.

\EIT

  \subsection{Approximation of Lipschitz-continuous functions}
In order to derive generalization bounds for target functions which are not sufficiently differentiable (and may not be in $\G_2$ or $\G_1$), we need to approximate any Lipschitz-continuous function, with a function $g \in \G_2$ with a norm $\gamma_2(g)$ that will grow as the approximation gets tighter. We give precise rates in the proposition below. Note the requirement for parity of the function $g$. The result below notably shows the density of $\G_1$ in uniform norm in the space of Lipschitz-continuous functions  of the given parity, which is already known since our activation functions are not polynomials~\citep{leshno1993multilayer}.

 \begin{proposition}[Approximation of Lipschitz-continuous functions on the sphere]
\label{prop:approx-sphere}
For $\delta $ greater than a constant depending only on $d$ and $\alpha$, for any function $g: \Sb^d \to \rb$  such that for all $x,y \in \Sb^d$,  $g(x) \leqslant \eta$ and $|g(x) - g(y) | \leqslant  \eta \| x - y\|_2$, and 
$g$ is even if $\alpha$ is odd (and vice-versa),
   there exists $h \in \G_2$, such that   $\gamma_2(h) \leqslant \delta$ and 
 $$\sup_{
x \in \Sb^d } | h(x) - g(x) | \leqslant C(d,\alpha) \eta  
\Big(\frac{\delta}{\eta}\Big)^{-1/(\alpha+(d-1)/2)} \log \Big(\frac{\delta}{\eta}\Big) 
.$$
\end{proposition}

 This proposition is shown in Appendix~\ref{app:appsphere2} for $d=1$ (using Fourier series) and in Appendix~\ref{app:proofapp} for all $d \geqslant 1$ (using spherical harmonics). We can make the following observations:
 \BIT
 \item[--] \emph{Dependence in $\delta$ and $\eta$}: as expected, the main term in the error bound $\big( {\delta}/{\eta}\big)^{-1/(\alpha+(d-1)/2)}$ is a decreasing function of $\delta / \eta$, that is when the norm $\gamma_2(h)$ is allowed to grow, the approximation gets tighter, and when the Lipschitz constant of $g$ increases, the approximation is less tight.
 
 \item[--] \emph{Dependence on $d$ and $\alpha$}: the rate of approximation is increasing in $d$ and $\alpha$. In particular the approximation properties are better for low $\alpha$.

 \item[--] \emph{Special case $d=1$ and $\alpha=0$}: up to the logarithmic term we recover the result of Prop.~\ref{prop:finites-sphere}, that is, the function $g$ is in $\G_2$.
 
 \item[--] \emph{Tightness}:  in   Appendix~\ref{app:tightness}, we provide a function which is not in the RKHS and for which the tightest possible approximation scales as
 $\delta^{-2 /( d/2 + \alpha - 2)}$. Thus the linear scaling of the rate as $d/2 + \alpha$ is not improvable (but constants are).
 \EIT

  \subsection{Linear functions}

 In this section, we consider a linear function on $\Sb^d$, that is $g(x) = v^\top x$ for a certain $v \in \Sb^d$, and compute its norm (or upper-bound thereof) both for $\G_1$ and $\G_2$, which is independent of $v$ and finite. In the following propositions, the notation $\approx$ means asymptotic equivalents when $d \to \infty$.
  \begin{proposition}[Norms of linear functions on the sphere]
\label{prop:linear-sphere}
Assume that $g: \Sb^d \to \rb$ is such $g(x) = v^\top x$ for a certain $v \in \Sb^d$. 
If $\alpha=0$, then $\gamma_1(g) \leqslant \gamma_2(g) = \frac{ 2 d \pi}{d-1} \approx 2 \pi$.
If $\alpha=1$, then $\gamma_1(g) \leqslant 2$, and for all $\alpha \geqslant 1$,
  $\gamma_1(g) \leqslant \gamma_2(g) = 
\frac{d}{d-1}\frac{4\pi}{\alpha}   \frac{\Gamma(\alpha/2+d/2+1)}{\Gamma(\alpha/2) \Gamma(d/2+1)} \approx C d^{\alpha/2}$.
\end{proposition}

We see that for $\alpha=1$, the $\gamma_1$-norm is less than a constant, and is much smaller than the $\gamma_2$-norm (which scales as $\sqrt{d}$). For $\alpha \geqslant 2$, we were not able to derive better bounds for $\gamma_1$ (other than the value of $\gamma_2)$.

\subsection{Functions of projections}

If $g(x)  = \varphi(w^\top x)$ for some unit-norm $w \in \rb^{d+1}$ and $\varphi$ a function defined on the real-line, then the value of the norms $\gamma_2$ and $\gamma_1$ differ significantly. Indeed, for $\gamma_1$,
we may consider a new variable $\tilde{x} \in \Sb^1 \subset \rb^2$, with its first component $\tilde{x}_1 = w^\top x$, and the function $\tilde{g}(x) = \varphi(\tilde{x}_1)$. We may then apply Prop.~\ref{prop:finites-sphere} to $\tilde{g}$ with $d=1$. That is, if $\varphi$ is $(\alpha+1)$-times differentiable with bounded derivatives, there exists a decomposition $\tilde{g}(\tilde{x}) = \int_{\Sb^1} \tilde{\mu}(\tilde{v}) \sigma( \tilde{v}^\top \tilde{x}) d \tilde{\mu}$, with $\gamma_1(\tilde{g}) = |\tilde{\mu}|(\Sb^1)$, which is \emph{not} increasing in $d$.
If we consider any vector $t \in \rb^{d+1}$ which is orthogonal to $w$ in $\rb^{d+1}$, then, we may define a measure $\mu$ supported in the circle defined by the two vectors $w$ and $t$ and which is equal to $\tilde{\mu}$ on that circle. The total variation of $\mu$ is the one of $\tilde{\mu}$ while $g$ can be decomposed using $\mu$ and thus $\gamma_1(g) \leqslant
\gamma_1(\tilde{g})$.  Similarly,
Prop.~\ref{prop:approx-sphere} could also be applied (and will for obtaining generalization bounds), also our reasoning works for any low-dimensional projections: the dependence on a lower-dimensional projection allows to reduce smoothness requirements.

However, for the RKHS norm $\gamma_2$, this reasoning does not apply.  For example,   a certain function~$\varphi$ exists, which is $s$-times differentiable, as shown in Appendix~\ref{app:tightness}, for  $s \leqslant \frac{d}{2} + \alpha$ (when $s-\alpha$ is even), and is not in $\G_2$. Thus, given  Prop.~\ref{prop:finites-sphere}, the dependence on a uni-dimensional projection does not make a difference regarding the level of smoothness which is required to belong to $\G_2$.

 \subsection{From the unit-sphere $\Sb^d$ to $\rb^{d+1}$}
 \label{sec:tosphere}
 We now extend the results above to functions defined on $\rb^d$, to be approximated by functions in $\F_1$ and $\F_2$.
More precisely, we first extend Prop.~\ref{prop:finites-sphere} and  Prop.~\ref{prop:approx-sphere}, and then consider norms of linear functions and functions of projections.

 \begin{proposition}[Finite variation]
 \label{prop:finite}
Assume that $f : \rb^d \to \rb$ is such that all $i$-th order derivatives exist and are upper-bounded on the ball
$\{ \| x\|_q \leqslant R\}$
 by $ \eta / R^{i}$ for $i \in \{0,\dots,k\}$,
where $s$ is the smallest integer such that $s \geqslant (d - 1)/2 + \alpha + 1$; then $f \in \F_2$ and $\gamma_2(f) \leqslant C(d,\alpha) \eta$, for  a constant $C(d,\alpha)$ that depends only on $d$ and $\alpha$.
\end{proposition}
\begin{proof}
By assumption, the function $x\mapsto f(Rx)$ has all its derivatives bounded by a constant times~$\eta$. Moreover, we have defined  $g(t,a) = f\big( \frac{Rt}{a} \big) a^\alpha  $ so that all derivatives are bounded by $\eta$. The result then follows immediately from Prop.~\ref{prop:finites-sphere}.
\end{proof}

\begin{proposition}[Approximation of Lipschitz-continuous functions]
\label{prop:approx}
For $\delta$ larger than a constant that depends only on $d$ and $\alpha$, for any function  $f : \rb^d \to \rb$  such that for all $x,y$ such that $\|x\|_q \leqslant R$ and $\| y\|_q \leqslant R$, $|f(x)| \leqslant  \eta$ and $|f(x) - f(y) | \leqslant \eta R^{-1} \|x-y\|_q$,   there exists $g \in \F_2$ such that   $\gamma_2(g) \leqslant \delta$ and 
$$\sup_{
\| x\|_q \leqslant R} | f(x) - g(x) | \leqslant C(d,\alpha)   \eta  
\Big(\frac{ \delta  }{\eta} \Big)^{-1/(\alpha+(d-1)/2)} \log \Big(\frac{ \delta  }{\eta} \Big) .
$$
\end{proposition}
\begin{proof}
With the same reasoning as above, we obtain that $g$ is Lipschitz-continuous with constant $\eta $, we thus get the desired approximation error from Prop.~\ref{prop:approx-sphere}.
\end{proof}

\paragraph{Linear functions.}
If $f(x) = w^\top x + b$, with $\|w\|_2 \leqslant \eta$ and $ b \leqslant \eta R$,
then for $\alpha=1$, it is straightforward that $\gamma_1(f) \leqslant 2 R\eta$. Moreover, we have $\gamma_2(f) \sim C R\eta$. For other values of $\alpha$, we also have $\gamma_1$-norms less than a constant (depending \emph{only} of $\alpha$) times $R\eta$. The RKHS norms are bit harder to compute since linear functions for $f$ leads to linear functions for $g$ only for $\alpha=1$.

\paragraph{Functions of projections.}
If $f(x) = \varphi(w^\top x)$ where $\|w\|_2\leqslant \eta$ and $\varphi:\rb \to \rb$ is a function, then the norm of $f$ is the same as the norm of the function $\varphi$ on the interval $[-R\eta, R\eta]$, and it thus does not depend on $d$. This is a consequence of the fact that the total mass of a Radon measure remains bounded even when the support has measure zero (which might not be the case for the RKHS defined in \mysec{rkhs}).
For the RKHS, there is no such results and it is in general not adaptive.

More generally, if $f(x) = \Phi(W^\top x)$ for $W \in \rb^{d \times s}$ with the largest singular value of $W$ less than~$\eta$, and $\Phi$ a function from $\rb^s$ to $\rb$, then for $\|x\|_2 \leqslant R$, we have $\| W^\top x\|_2 \leqslant R\eta$, and thus we may apply our results for $d=s$.

\paragraph{$\ell_1$-penalty on input weights ($p$=1).}
When using an $\ell_1$-penalty on input weights instead of an $\ell_2$-penalty, the results in Prop.~\ref{prop:finite} and \ref{prop:approx} are unchanged (only the constants that depend on $d$ are changed). Moreover, when $\| x\|_\infty \leqslant 1$ almost surely, functions of the form $f(x) = \varphi(w^\top x)$ where $\|w\|_1\leqslant \eta$ and $\varphi:\rb \to \rb$ is a function, will also inherit from properties of $\varphi$ (without any dependence on dimension). Similarly, for functions of the form $f(x) = \Phi(W^\top x)$ for $W \in \rb^{d \times s}$ with all columns of $\ell_1$-norm less than $\eta$, we have $\| 
W^\top x\|_\infty \leqslant R\eta$ and we can apply the $s$-dimensional result.

\subsection{Related work}
\label{sec:related}

In this section, we show how our results from the previous sections relate to existing work on neural network approximation theory.

\paragraph{Approximation of Lipschitz-continuous functions with finitely many neurons.} In this section, we only consider the case $\alpha=1$, for which we have two approximation bounds: Prop.~\ref{prop:approx} which  approximates any $\eta$-Lipschitz-continuous function by a function with finite $\gamma_1$-norm less than $\delta$ and uniform error less than $\displaystyle \eta  
\big({ \delta  }/{\eta} \big)^{-2/(d+1)} \log \big({ \delta  }/{\eta} \big) $, and Prop.~\ref{prop:neurons} which shows that a function with $\gamma_1$-norm less than $\delta$, may be approximated with $r$ neurons with uniform error $\delta r  ^{- ( d +3 )/(2d)}$.

Thus, given $r$ neurons, we get an approximation of the original function with uniform error
 $$
 \eta  
\big({ \delta  }/{\eta} \big)^{-2/(d+1)} \log \big({ \delta  }/{\eta} \big) 
+\delta r  ^{- ( d +3 )/(2d)}.
 $$
 We can optimize over $\delta$, and use $\delta = \eta n^{(d+1)/(2d)}$, to obtain a uniform approximation bound proportional to $ {\eta (\log n)}{n^{-1/d}}$, for approximating an $\eta$-Lipschitz-continuous function with $n$ neurons.

\paragraph{Approximation by ridge functions.} The approximation properties of single hidden layer neural networks have been studied extensively, where they are often referred to as ``ridge function'' approximations. As shown by \citet[Corollary 6.10]{pinkus1999approximation}---based on a result from~\citet{petrushev1998approximation}, the approximation order of $n^{-1/d}$ for the rectified linear unit was already known, but only in $L_2$-norm (and without the factor $\log n$), and without any constraints on the input and output weights. In this paper, we provide an explicit control of the various weights, which is needed for computing estimation errors.
Moreover, while the two proof techniques use spherical harmonics, the proof of~\citet{petrushev1998approximation} relies on quadrature formulas for the associated Legendre polynomials, while ours relies on the relationship with the associated positive definite kernels, is significantly simpler, and offers additional insights into the problem (relationship with convex neural networks and zonoids). \citet[Theorem 2.3]{maiorov2006approximation} also derives a similar result, but in  $L_2$-norm (rather than uniform norm), and for sigmoidal activation functions (which are bounded). Note finally, that the order $O(n^{-1/d})$ cannot be improved~\citep[][Theorem 4.2]{devore1989optimal}. Also, \citet[Theorem 5]{maiorov2000near} derive similar upper and lower bounds based on a random sampling argument which is close to using random features in the RKHS setting described in \mysec{RKHS}.

\paragraph{Relationship to hardness results for Boolean-valued functions.} 
In this paper, we consider a particular view of the curse of dimensionality and ways of circumventing it, that is, our distribution over inputs is arbitrary, but our aim is to approximate a  {real-valued} function. Thus, all hardness results depending on functions with values in $\{0,1\}$ do not apply there directly---see, e.g., \citet[Chapter 20]{shaibook}, for the need of exponentially many hidden units for approximating most of the functions from $\{0,1\}^d$ to $\{0,1\}$. 

Our approximation bounds show that, without any assumption beyond Lipschitz-continuity of the target function, it sufficient to have a number of hidden units which is still exponential in dimension (hence we also suffer from the curse of dimensionality), but a soon as the target function depends on linear low-dimensional structure, then we lose this exponential dependence. It would be interesting to study an extension to $\{0,1\}$-valued functions, and also to relate our results to the number of linear regions delimited by neural networks with rectified linear units~\citep{montufar2014number}.

 \section{Generalization bounds}
 \label{sec:adaptivity}
 \label{sec:bounds}
 
 Our goal is to derive the generalization bounds outlined in \mysec{learning-losses} for neural networks with a single hidden layer. The main results that we obtain are summarized in Table~\ref{tab:table} and show adaptivity to assumptions that avoid the curse of dimensionality.
 
More precisely, given some  distribution over the pairs $(x,y) \in \X\times \Y$, a loss function $\ell: \Y \times \rb \to \rb$, our aim is to find a function $f: \rb^d \to \rb$ such that $J(f) = \E \big[ \ell(y,f(x)) \big]$ is small, given some i.i.d.~observations $(x_i,y_i)$, $i=1,\dots,n$. We consider the empirical risk minimization framework over a space of functions $\F$, equipped with a norm $\gamma$ (in our situations, $\F_1$ and $\F_2$, equipped with $\gamma_1$ or $\gamma_2$). The empirical risk $\hat{J}(f) = \frac{1}{n} \sum_{i=1}^n \ell(y_i,f(x_i))$, is minimized by constraining $f$ to be in the ball $\mathcal{F}^\delta = \{ f \in \F, \ \gamma(f) \leqslant \delta \}$.
 
 We  assume that almost surely, $\| x\|_q \leqslant R$,  that for all $y$ the function $u \mapsto \ell(y,u)$ is $G$-Lipschitz-continuous on $\{ |u| \leqslant \sqrt{2}  \delta \}$, and that almost surely, $\ell(y,0) \leqslant G  \delta $. 
As before $z$ denotes $z=(x^\top, R)^\top$ so that $\| z\|_q \leqslant \sqrt{2}   R$.
This corresponds to the following examples:

\vspace*{-.125cm}

\BIT
\item[--] Logistic regression and support vector machines: we have $G = 1$.
\item[--] Least-squares regression: we take $G = \max\big\{ \sqrt{2}  \delta + \| y\|_\infty , \frac{\| y\|_\infty^2}{\sqrt{2}    \delta} \big\}$.  \EIT

  Approximation errors  $\inf_{f\in \F^\delta } J(f) - \inf_{f\in \F} J(f)$ will be obtained from the approximation results from \mysec{approx} by assuming that the optimal target function $f_\ast$ has a specific form. Indeed, we have:
  $$\inf_{f\in \F^\delta } J(f) -   J(f_\ast)
  \leqslant G \inf_{f\in \F^\delta } \Big\{ \sup_{\| x\|_q \leqslant R} | f(x) - f_\ast(x) | \Big\}.$$

  We now deal with estimation errors $   \sup_{ f \in \F^\delta} | \hat{J}(f) - J(f) | $ using Rademacher complexities.

 \subsection{Estimation errors and Rademacher complexity}
 \label{sec:rademacher}

The following proposition  bounds the uniform deviation between $J$ and its empirical counterpart~$\hat{J}$. This result is standard~\citep[see, e.g.,][]{koltchinskii2001rademacher,bartlett2003rademacher} and may be extended in bounds that hold with high-probability.

\begin{proposition}[Uniform deviations]
\label{prop:rademacher}
We have the following bound on the expected uniform deviation:
$$\E \bigg[ \sup_{ \gamma_1(f) \leqslant \delta   } | J(f) - \hat{J}(f)| \bigg]
\leqslant 4 \frac{G \delta}{\sqrt{n}} C(p,d,\alpha),$$
with the following constants:
\BIT
\item[--] for $\alpha \geqslant 1$, $C(p,d ,\alpha) \leqslant \alpha \sqrt{ 2 \log (d+1)}$ for $p=1$ and $C(p,d,\alpha) \leqslant \frac{\alpha}{\sqrt{p-1}}$ for $p \in (1,2]$ 
\item[--]  for $\alpha =0$, $C(p,d ,\alpha) \leqslant C\sqrt{d+1}$, where $C$ is a universal constant.
\EIT
\end{proposition}

\begin{proof}
We use the standard framework of Rademacher complexities and get:
\BEAS
& & \E    \sup_{ \gamma_1(f) \leqslant \delta   } | J(f) - \hat{J}(f)| \\
& \leqslant & 2 \E \sup_{\gamma_1(f) \leqslant \delta} \bigg| \frac{1}{n} \sum_{i=1}^n \tau_i \ell(y_i,f(x_i) ) \bigg| \mbox{ using Rademacher random variables } \tau_i, \\
& \leqslant & 2 \E \sup_{\gamma_1(f) \leqslant \delta} \bigg| \frac{1}{n} \sum_{i=1}^n \tau_i \ell(y_i, 0 )  \bigg|
+ 2 \E \sup_{\gamma_1(f) \leqslant \delta} \bigg| \frac{1}{n} \sum_{i=1}^n \tau_i \big[ \ell(y_i, f(x_i) ) - \ell(y_i,0) \big]\bigg| 
\\
& \leqslant & 2 \frac{G \delta}{\sqrt{n}}
+ 2 G \E \sup_{\gamma(f) \leqslant \delta  } \bigg| \frac{1}{n} \sum_{i=1}^n \tau_i  f(x_i)   \bigg| 
\mbox{ using the Lipschitz-continuity of the loss},\\
& \leqslant & 2 \frac{G \delta}{\sqrt{n}}
+ 2 G  {\delta}  \E \sup_{ \| v\|_p \leqslant 1/R } \bigg| \frac{1}{n} \sum_{i=1}^n \tau_i   (v^\top z_i)_+ ^\alpha  \bigg| \mbox{ using \eq{incr}}.
\EEAS
We then take different routes for $\alpha \geqslant 1$ and $\alpha = 0$. 

For $\alpha \geqslant 1$, we have the upper-bound
\BEAS
\E    \sup_{ \gamma_1(f) \leqslant \delta   } | J(f) - \hat{J}(f)| & \leqslant  & 2 \frac{G \delta}{\sqrt{n}}
+ 2  {G\delta  \alpha}  \E \sup_{  \| v\|_p \leqslant 1/R } \bigg| \frac{1}{n} \sum_{i=1}^n \tau_i    v^\top z_i    \bigg|   \\
& & \hspace*{4cm} \mbox{ using the $\alpha$-Lipschitz-cont.~of } (\cdot)_+^\alpha \mbox{ on } [-1,1],
\\
& \leqslant & 2 \frac{G \delta}{\sqrt{n}}
+ 2 \frac{ G \alpha\delta}{R n} \E  \bigg\|   \sum_{i=1}^n \tau_i      z_i    \bigg\|_q .
\EEAS

From~\citet{kakade2009complexity}, we get the following bounds on Rademacher complexities:
\BIT
\item[--]
If $ p \in (1,2]$, then $q \in [2,\infty)$, and
$\E  \big\|   \sum_{i=1}^n \tau_i      z_i    \big\|_q \leqslant  {\sqrt{q-1}} R\sqrt{n}
=  \frac{1}{\sqrt{p-1}} R\sqrt{n}$

\item[--]
If $p=1$, then $q = \infty$,  and
$ \E  \big\|   \sum_{i=1}^n \tau_i      z_i    \big\|_q \leqslant R\sqrt{n} \sqrt{ 2 \log(d+1) }$.

\EIT

Overall, we have $\E  \big\|   \sum_{i=1}^n \tau_i      z_i    \big\|_q  \leqslant \sqrt{n} R C(p,d)$ with $C(p,d)$ defined above, and thus
$$
\E    \sup_{ \gamma(f) \leqslant \delta } | J(f) - \hat{J}(f)|
\leqslant 2 \frac{G \delta}{\sqrt{n}} ( 1 + \alpha C(p,d)) \leqslant 4 \frac{G \delta \alpha}{\sqrt{n}} C(p,d).$$

For $\alpha =0$, we can simply go through the VC-dimension of half-hyperplanes, which is equal to $d$, and Theorem 6 from~\citet{bartlett2003rademacher}, that shows that
$ \E \sup_{ v \in \rb^{d+1} } \bigg| \frac{1}{n} \sum_{i=1}^n \tau_i   1_{v^\top z_i}  \bigg|
\leqslant C \frac{\sqrt{d+1}}{\sqrt{n}}$, where $C$ is a universal constant.  

Note that using standard results from Rademacher complexities, we have, with probability greater than $1-u$,
$ \displaystyle
  \sup_{ \gamma_1(f) \leqslant \delta } | J(f) - \hat{J}(f)|
\leqslant   \E    \sup_{ \gamma_1(f) \leqslant \delta } | J(f) - \hat{J}(f)| +  \frac{2 G \delta}{\sqrt{n}} \sqrt{\log \frac{2}{u}}  
$.
\end{proof}

 \subsection{Generalization bounds for $\ell_2$-norm constraints on input weights ($p=2$)}

We now provide generalization bounds for the minimizer of the empirical risk given the contraint that $\gamma_1(f) \leqslant \delta $ for a well chosen $\delta$, that will depend on the assumptions regarding the target function~$f_\ast$, listed in \mysec{intro}. In this section, we consider an $\ell_2$-norm on input  weights $w$, while in the next section, we consider the $\ell_1$-norm. The two situations are summarized  and compared in Table~\ref{tab:table}, where we consider that $\| x\|_\infty \leqslant r$ almost surely, which implies that our bound $R$ will depend on dimension as $R \leqslant r\sqrt{d}$.

Our   generalization bounds are expected values of the excess expected risk for a our estimator (where the expectation is taken over the data).

\paragraph{Affine functions.} We assume $f_\ast(x) = w^\top x +b $, with $\|w\|_2 \leqslant \eta$  and 
$|b| \leqslant R \eta$. Then,  as seen in \mysec{tosphere}, $f_\ast \in \F_1$ with $\gamma_1(f_\ast) \leqslant C(\alpha) \eta R$ (the constant is independent of $d$ because we approximate an affine function).
From Prop.~\ref{prop:rademacher}, we thus get a generalization bound proportional to $\frac{GR\eta}{\sqrt{n}}$ times a constant (that may depend on $\alpha$), which is the same as assuming directly that we optimize over linear predictors only. 
The chosen $\delta$ is then a constant times $R \eta$, and does not grow with $n$, like in parametric estimation (although we do use a non-parametric estimation procedure).

 \paragraph{Projection pursuit.} We assume $f_\ast(x) = \sum_{j=1}^k f_j(w_j^\top x )$, with 
$\|w_j\|_2 \leqslant \eta$ and  each $f_j$ bounded by $\eta R$ and $1$-Lipschitz continuous. From Prop.~\ref{prop:approx}, we may approach each $x \mapsto  f_j(w_j^\top x )$ by a function with $\gamma_1$-norm less than $\delta \eta R$ and uniform approximation $C(\alpha) \eta R \delta^{-1/\alpha} \log \delta$. This leads to a total approximation error of $k C(\alpha) G \eta R \delta^{-1/\alpha} \log \delta$ for a norm
less than $k \delta \eta R$ (the constant is independent of $d$ because we approximate a function of one-dimensional projection).

For $\alpha \geqslant 1$, from Prop.~\ref{prop:rademacher}, the estimation error is 
$\frac{kGR\eta \delta}{\sqrt{n}}$, with an overall bound
of $ C(\alpha) kGR\eta \big( \frac{\delta}{\sqrt{n}} + \delta^{-1/\alpha} \log \delta   \big)$.
 With $\delta = n^{\alpha/2(\alpha+1)}$ (which grows with $n$), we get an optimized generalization bound of
 $ C(\alpha)  kGR\eta \frac{ \log n }{ n^{1/(2\alpha+2)} }$, with a scaling independent of the dimension $d$ (note however that $R$ typically grow with $\sqrt{d}$, i.e., $r\sqrt{d}$, if we have a bound in $\ell_\infty$-norm for all our inputs $x$).
 
 For $\alpha = 0$, from Prop.~\ref{prop:finite}, the target function belongs to $\F_1$ with a norm less than
 $ k GR \eta$, leading to an overall generalization bound of     $\frac{kGR\eta \sqrt{d} }{\sqrt{n}}$. 

Note that when the functions $f_j$ are exactly the activation functions, the bound is better, as these functions directly belong to the space $\F_1$.

\paragraph{Multi-dimensional projection pursuit.} We extend the situation above, by assuming $f_\ast(x) = \sum_{j=1}^k F_j(W_j^\top x )$  with each $W_j \in \rb^{d \times s}$ having all singular values less than $\eta$ and each $F_j$
 bounded by $\eta R$ and $1$-Lipschitz continuous. From Prop.~\ref{prop:approx}, we may approach each $x \mapsto  F_j(W_j^\top x )$ by a function with $\gamma_1$-norm less than $\delta \eta R$ and uniform approximation $C(\alpha,s) \eta R \delta^{-1/( \alpha+(s-1)/2)} \log \delta$. This leads to a total approximation error of $k C(\alpha,s) G \eta R  \delta^{-1/( \alpha+(s-1)/2)} \log \delta$.

For $\alpha \geqslant 1$, the estimation error is 
$ {kGR\eta \delta}/{\sqrt{n}}$, with an overall bound
of $ C(\alpha,s) kGR\eta \big( {\delta}/{\sqrt{n}} +  \delta^{-1/( \alpha+(s-1)/2)} \log \delta \big)$. With $\delta = n^{(\alpha+(s-1)/2)/ (2\alpha+s-1)}$, we get an optimized bound of
 $ \frac{ C(\alpha,s) kGR\eta }{ n^{1/(2\alpha+s+1)} }\log n$.
 
 For $\alpha = 0$, we have 
  an overall bound
of $ C(s) kGR\eta \big( \delta^{-2/(s-1)} \log \delta + \frac{\delta \sqrt{d}}{\sqrt{n}}\big)$, and with 
$\delta = (n/d)^{(s-1)/(s+1)}$, we get a generalization bound scaling as
 $ \frac{ C(s) kGR\eta }{ (n/d)^{1/(s+1)} }\log (n/d)$.

Note that for $s=d$ and $k=1$, we recover the usual Lipschitz-continuous assumption, with a rate of
 $ \frac{ C(\alpha,d) kGR\eta }{ n^{1/(2\alpha+d+1)} }\log n$.

\begin{table}
\hspace*{-.25cm}
\begin{tabular}{|c|c|c|c|}
\hline
\textcolor{white}{$ \Big|$} function space & $\| \cdot \|_2$, $\alpha \geqslant 1$ &  $\| \cdot \|_1$, $\alpha \geqslant 1$  & $\alpha=0$ \\
\hline
\textcolor{white}{$ \bigg|$}  $w^\top x + b$ &  $\frac{\displaystyle d^{1/2}}{\displaystyle n^{1/2}}$ &$\sqrt{q} \big( \frac{ \displaystyle \log d}{ \displaystyle n}\big)^{1/2}$ & $ 
 \frac{\displaystyle    ( dq )^{1/2}   }{ \displaystyle n^{1/2}} $\\
\hline
\textcolor{white}{$ \bigg|$}  $\displaystyle  \!\!\!\!\!\! \sum_{j=1}^k f_j( w_j^\top x)$, $w_j \in \rb^d \!\!\!$ &
$ 
 \frac{\displaystyle  k   d^{1/2}}{ \displaystyle n^{1/(2\alpha+2)}} \log n $ 
  & $ 
 \frac{\displaystyle  k   q^{1/2}  (\log d)^{ 1/(\alpha+1)} }{ \displaystyle n^{1/(2\alpha+2)}} \log n $ &   $ 
 \frac{\displaystyle  k  ( dq )^{1/2}    }{ \displaystyle n^{1/2}} $ \\
\hline
\textcolor{white}{$ \bigg|$}  $ \displaystyle \!\!\!\!\! \sum_{j=1}^k f_j( W_j^\top x)$, $W_j \in \rb^{d \times s}\!\! $ &
$  \!\!
 \frac{\displaystyle  k  d^{1/2} }{ \displaystyle n^{1/(2\alpha+s+1)}} \log n \!\! $ 
  & $ \!
 \frac{\displaystyle  k   q^{1/2}  (\log d)^{ 1/(\alpha+(s+1)/2)} }{ \displaystyle n^{1/(2\alpha+s+1)}} \log n \! \!$
& $ \!\!
 \frac{\displaystyle   ( dq )^{1/2}  d^{  1/(s+1)} }{ \displaystyle n^{1/( s+1)}} \log n \!\! $  \\
\hline
\end{tabular}
\caption{Summary of generalization bounds with different settings. See text for details.}
\label{tab:table}
\end{table}

\vspace*{.5cm}

We can make the following observations:
\BIT

\item[--] \emph{Summary table}: when we know a bound $r$ on all dimensions of $x$, then we may take $R = r \sqrt{d}$; this is helpful in comparisons in Table~\ref{tab:table}, where $R$ is replaced by $r \sqrt{d}$ and the dependence in $r$ is removed as it is the same for all models.

\item[--] \emph{Dependence on $d$}: when making only a global Lipschitz-continuity assumption, the generalization bound has a bad scaling in $n$, i.e.,  as $n^{-1/(2\alpha+d+1)}$, which goes down to zero slowly when $d$ increases. However, when making structural assumptions regarding the dependence on unknown lower-dimensional subspaces, the scaling in $d$ disappears.

\item[--] \emph{Comparing different values of $\alpha$}: the value $\alpha=0$ always has the best scaling
in $n$, but constants are better for $\alpha \geqslant 1$ (among which $\alpha=1$ has the better scaling in $n$).

\item[--] \emph{Bounds for $\F_2$}: The simplest upper bound for the penalization by the space $\F_2$ depends on the approximation properties of $\F_2$. For linear functions and $\alpha=1$, it is less than $\sqrt{d} \eta R$, with a bound
 $\frac{GR\eta \sqrt{d}}{\sqrt{n}}$.  For the other values of $\alpha$, there is a constant $C(d)$. Otherwise, there is no adaptivity and all other situations only lead to upper-bounds of $O(n^{-1/(2\alpha + d+ 1)})$. See more details in \mysec{comp}.
 
 \item[--] \emph{Sample complexity}: Note that the generalization bounds above may be used to obtain sample complexity results such as
$d \varepsilon^{-2}$ for affine functions, $( \varepsilon k^{-1} d^{-1/2} )^{-2\alpha-2}$ for projection pursuit, and 
$
( \varepsilon k^{-1} d^{-1/2} ) ^{ -s-1-2\alpha}
$ for the generalized version (up to logarithmic terms).

\item[--] \emph{Relationship to existing work}: \citet[Theorem 1.1]{maiorov2006approximation} derives similar results for neural networks with sigmoidal activation functions (that tend to one at infinity) and the square loss only, and for a level of smoothness of the target function which grows with dimension (in this case, once can get easily rates of $n^{-1/2}$).  Our result holds for problems where only  bounded first-order derivatives are assumed, but by using Prop.~\ref{prop:finites-sphere}, we would get similar rate by ensuring the target function belongs to $\F_2$ and hence to $\F_1$.

\EIT

\paragraph{Lower bounds.}
 In the sections above, we have only provided generalization bounds. Although interesting, deriving lower-bounds for the generalization performance when the target function belongs to certain function classes is out of the scope of this paper. Note however, that results from~\citet{karthik} suggest that the Rademacher complexities of the associated function classes provide such lower-bounds. For general Lipschitz-functions, these Rademacher complexities decreases as $n^{- \max\{d,2\}}$ \citep{luxburg2004distance}.

 \subsection{Generalization bounds for $\ell_1$-norm constraints on input weights ($p=1$)}
\label{sec:highdim}
We consider the same three situations, assuming that linear predictors have at most $q$ non-zero elements. We assume that each component of $x$ is almost surely bounded by $r$ (i.e., a bound in $\ell_\infty$-norm).

\paragraph{Affine functions.} We assume $f_\ast(x) = w^\top x +b $, with $\|w\|_2 \leqslant \eta$  and 
$|b| \leqslant R \eta$. Given that we have assumed that $w$ has at most $q$ non-zeros, we have $\|w\|_1 \leqslant \sqrt{q} \eta$.

Then, $f_\ast \in \F_1$ with $\gamma_1(f) \leqslant C(\alpha) \eta r \sqrt{q}$, with a constant that is independent of $d$ because we have an affine function.

From Prop.~\ref{prop:rademacher},
we thus get a rate of $\frac{Gr\eta\sqrt{q \log(d)}}{\sqrt{n}}$ times a constant (that may depend on $\alpha$), which is the same as assuming directly that we optimize over linear predictors only~\citep[see, for example,][]{buhlmann2011statistics}. We recover a high-dimensional phenomenon (although with a slow rate in $1/\sqrt{n}$), where $d$ may be much larger than $n$, as long as $\log d$ is small compared to $n$.
The chosen $\delta$ is then a constant times $r \eta \sqrt{q}$ (and does not grow with $n$).

\paragraph{Projection pursuit.} We assume $f_\ast(x) = \sum_{j=1}^k f_j(w_j^\top x )$, with 
$\|w_j\|_2 \leqslant \eta$ (which implies $\|w_j\|_1 \leqslant \sqrt{q} \eta$ given our sparsity assumption) and  each $f_j$ bounded by $\eta r \sqrt{q}$ and $1$-Lipschitz continuous. We may approach each $x \mapsto  f_j(w_j^\top x )$ by a function with $\gamma_1$-norm less than $\delta \eta r \sqrt{q}$ and uniform approximation $C(\alpha) \eta r \sqrt{q} \delta^{-1/\alpha} \log \delta$, with a constant that is independent of $d$ because we have a function of one-dimensional projection. This leads to a total approximation error of $k C(\alpha) G \eta r \sqrt{q} \delta^{-1/\alpha} \log \delta$ for a norm
less than $k \delta \eta r\sqrt{q} $.

For $\alpha \geqslant 1$, the estimation error is 
$\frac{kGr \eta \delta   \sqrt{q \log d}}{\sqrt{n}}$, with an overall bound
of $ C(\alpha) kGr \sqrt{q} \eta \big( \delta^{-1/\alpha} \log \delta + \frac{\delta\sqrt{\log d}}{\sqrt{n}}\big)$. With $\delta = (n/\log d) ^{\alpha/2(\alpha+1)}$, we get an optimized bound of
 $ C(\alpha)  kGr \sqrt{q} \eta \frac{ \log n (\log d)^{1/(2\alpha+2)} }{ n^{1/(2\alpha+2)} }$, with a scaling only dependent in $d$ with a logarithmic factor.

 For $\alpha = 0$, the target function belongs to $\F_1$ with a norm less than
 $ k Gr \sqrt{q} \eta$, leading to an overal bound of     $\frac{kGr \eta \sqrt{q \log d} }{\sqrt{n}}$ (the sparsity is not helpful in this case).

\paragraph{Multi-dimensional projection pursuit.} We assume 
 $f_\ast(x) = \sum_{j=1}^k F_j(W_j^\top x )$  with each $W_j \in \rb^{d \times s}$, having all columns with $\ell_2$-norm less than $\eta$ (note that this is a weaker requirement than having all singular values that are less than $\eta$). If we assume that each of these columns has at most~$q$ non-zeros, then the $\ell_1$-norms are less than $r \sqrt{q}$ and we may use the approximation properties described at the end of \mysec{tosphere}. We also assume that each $F_j$
 is bounded by $\eta r \sqrt{q}$ and $1$-Lipschitz continuous (with respect to the $\ell_2$-norm).   
 
 We may approach each $x \mapsto  F_j(W_j^\top x )$ by a function with $\gamma_1$-norm less than $\delta \eta r \sqrt{q}$ and uniform approximation $C(\alpha,s) \eta r \sqrt{q} \delta^{-1/( \alpha+(s-1)/2)} \log \delta$. This leads to a total approximation error of $k C(\alpha,s) G \eta r \sqrt{q}  \delta^{-1/( \alpha+(s-1)/2)} \log \delta$.

For $\alpha \geqslant 1$, the estimation error is 
$ {kGr\sqrt{q}\eta \delta \sqrt{ \log d }}/{\sqrt{n}}$, with an overall bound which is equal to
 $ C(\alpha,s) kGr\sqrt{q}\eta \big( \delta^{-1/( \alpha+(s-1)/2)} \log \delta + \frac{\delta \sqrt{\log d}}{\sqrt{n}}\big)$. With $\delta = (n/\log d)^{(\alpha+(s-1)/2)/ (2\alpha+ s-1)}$, we get an optimized bound of
$ \displaystyle \frac{ C(\alpha,s) kGr\sqrt{q}\eta (\log d)^{1/(2\alpha+s+1)}}{ n^{1/(2\alpha+s+1)} }\log n$.
 
 For $\alpha = 0$, we have the bound
 $ \frac{ C(s) kGr \sqrt{q}\eta }{ (n/d)^{1/(s+1)} }\log (n/d)$, that is we cannot use the sparsity as the problem is invariant to the chosen norm on hidden weights.

\vspace*{.5cm}

We can make the following observations:

\vspace*{-.125cm}

\BIT

\item[--] \emph{High-dimensional variable selection}: when $k=1$,   $s = q$ and $W_1$ is a projection onto $q$ variables, then we obtain a bound proportional to
$ \frac{ \sqrt{q}\eta (\log d)^{1/(2\alpha+s+1)}}{ n^{1/(2\alpha+s+1)} }\log n $, which exhibits a high-dimensional scaling in a non-linear setting. Note that beyond sparsity, no assumption is made (in particular regarding correlations between input variables), and we obtain a high-dimensional phenomenon where $d$ may be much larger than $n$.

\item[--] \emph{Group penalties}: in this paper, we only consider $\ell_1$-norm on input weights; when doing joint variable selection for all basis functions, it may be worth using a group penalty~\citep{yuan2006model,grouplasso}.
\EIT

\subsection{Relationship to kernel methods and random sampling}
\label{sec:comp}

The results presented in the two sections above were using the space $\F_1$, with an $L_1$-norm on the outputs weights (and either an $\ell_1$- or $\ell_2$-norm on input weights). As seen in Sections~\ref{sec:RKHS} and \ref{sec:kersigma}, when using an $L_2$-norm on output weights, we obtain a reproducing kernel Hilbert space $\F_2$. 

As shown in \mysec{approximation}, the space $\F_2$ is significantly smaller than $\F_1$, and in particular is not adaptive to low-dimensional linear structures, which is the main advantage of the space $\F_1$. However, algorithms for $\F_2$ are significantly more efficient, and there is no need for the conditional gradient algorithms presented in \mysec{cg}. The first possibility is to use the usual RKHS representer theorem with the kernel functions computed in \mysec{kersigma}, leading to a computation complexity of $O(n^2)$. Alternatively, as shown by~\citet{rahimi2007random}, one may instead sample $m$ basis functions that is $m$ different hidden units, keep the input weights fixed and optimize only the output layer with a squared $\ell_2$-penalty. This will quickly (i.e., the error goes down as $1/\sqrt{m}$) approach the non-parametric estimator based on penalizing by the RKHS norm $\gamma_2$. Note that this argument of random sampling has been used to study approximation bounds for neural networks with finitely many units~\citep{maiorov2000near}.

Given the usage of random sampling with $L_2$-penalties, it is thus tempting to sample weights, but now optimize an $\ell_1$-penalty, in order to get the non-parametric estimator obtained from penalizing by $\gamma_1$. When the number of samples $m$ tends to infinity, we indeed obtain an approximation  that converges to $\gamma_1$ (this is simply a uniform version of the law of large numbers). However, the rate of convergence does depend on the dimension $d$, and in general exponentially many samples would be needed for a good approximation---see~\citet[Section 6]{siammatrix} for a more precise statement in the related context of convex matrix factorizations.

 \subsection{Sufficient condition for polynomial-time algorithms}
 \label{sec:oracle}
 \label{sec:sufficient}

 In order to preserve the generalization bounds presented above, it is sufficient to be able to solve the following problem, for any $y \in \rb^n$ and $z_1,\dots,z_n \in \rb^{d+1}$:
 \BEQ
\label{eq:incrGGG}
   \sup_{ \| v\|_p=1 }  \bigg| \frac{1}{n} \sum_{i=1}^n y_i (v^\top z_i)_+^\alpha \bigg| ,
\EEQ
\emph{up to a constant factor}. That is, there exists $\kappa \geqslant 1$, such that for all $y$ and $z$, we may compute  $\hat{v}$ such that $\|\hat{v}\|_p=1$ and 
$$
 \bigg| \frac{1}{n} \sum_{i=1}^n y_i (\hat{v}^\top z_i)_+^\alpha \bigg| \geqslant 
 \frac{1}{\kappa}
 \sup_{ \| v\|_p=1 }  \bigg| \frac{1}{n} \sum_{i=1}^n y_i (v^\top z_i)_+^\alpha \bigg|.$$
This is provably NP-hard for $\alpha=0$ (see \mysec{FW1}), and  for $\alpha = 1$  (see \mysec{FW2}). If such an algorithm is available, the approximate conditional gradient presented in \mysec{condgrad} leads to an estimator with the same generalization bound. Moreover, given the strong hardness results for improper learning in the situation $\alpha=0$~\citep{klivans2006cryptographic,livni2014computational}, a convex relaxation that would consider a larger set of predictors (e.g., by relaxing $vv^\top$ into a symmetric positive-definite matrix), and obtained a constant approximation guarantee, is also ruled out.

However, this is only a sufficient condition, and a simpler sufficient condition may be obtained. In the following, we consider $\V = \{ v \in \rb^{d+1}, \ \| v\|_2 = 1\}$ and basis functions $\varphi_v(z) = (v^\top z)_+^\alpha$ (that is we specialize to the $\ell_2$-norm penalty on weight vectors). We consider a new variation norm $\hat{\gamma}_1$ which has to satisfy the following assumptions:
\BIT
\item[--] \emph{Lower-bound on $\gamma_1$}: It is defined from functions $\hat{\varphi}_{\hat{v}}$, for $\hat{v} \in \hat{\V}$, where for any $v \in \V$, there exists $\hat{v} \in \hat{\V}$ such that $\varphi_v = \hat{\varphi}_{\hat{v}}$. This implies that 
the corresponding space $\hat{\F}_1$ is larger than $\F_1$ and that if $f \in \mathcal{F}_1$, then $\hat{\gamma}_1(f) \leqslant \gamma_1(f)$.
\item[--] \emph{Polynomial-time algorithm for dual norm}: The dual norm  
  $\displaystyle \sup_{\hat{v} \in \hat{\V}}  \bigg| \frac{1}{n} \sum_{i=1}^n y_i \hat{\varphi}_{\hat{v}}(z_i)  \bigg|$
may be computed in polynomial time.
\item[--] \emph{Performance guarantees for random direction}: There exists $\kappa>0$, such that   for any vectors $z_1,\dots,z_n \in \rb^{d+1}$ with $\ell_2$-norm less than $R$, and random standard Gaussian vector $y \in \rb^{n}$,
\BEQ
\label{eq:incrG}
\sup_{\hat{v} \in \hat{\V}}  \bigg| \frac{1}{n} \sum_{i=1}^n y_i \hat{\varphi}_{\hat{v}}(x_i)  \bigg|  \leqslant  \kappa \frac{R}{\sqrt{n}}.
 \EEQ
 We may also replace the standard Gaussian vectors by Rademacher random variables.
\EIT

We can then penalize by $\hat{\gamma}$ instead of $\gamma$. Since $\hat{\gamma}_1\leqslant \gamma_1$, approximation properties are transferred, and because of the result above, the Rademacher complexity for $\hat{\gamma}_1$-balls scales as well as for $\gamma_1$-balls.
In the next section, we show convex relaxations which cannot achieve these and leave the existence or non-existence of such norm $\hat{\gamma}_1$ as an open problem.

  \section{Convex relaxations of the Frank-Wolfe step}
 \label{sec:approximation}
 
   \label{sec:gauge}

 In this section, we provide approximation algorithms for the following problem of maximizing, for a given
 $y \in \rb^n$ and vectors $z_1,\dots,z_n$:
 $$
   \sup_{ \| v\|_p=1 }  \frac{1}{n} \sum_{i=1}^n y_i (v^\top z_i)_+^\alpha 
 $$

These approximation algorithms may be divided in three families, as they may be based on (a) geometric interpretations as linear binary classification or computing Haussdorff distances (see \mysec{FW1} and \mysec{FW2}), (b) on direct relaxations, on (c)  relaxations of sign vectors. For simplicity, we only focus on the case $p=2$ (that is $\ell_2$-constraint on weights) and on $\alpha=1$ (rectified linear units). As described in \mysec{oracle}, constant-factor approximation ratios are not possible, while approximation ratios that increases with $n$ are possible (but as of now, we only obtain scalings in $n$ that provide a provable sample complexity with a polynomial time algorithm 
  which is exponential in the dimension $d$.

\subsection{Semi-definite programming relaxations}

We present two relaxations, which are of the form described in \mysec{oracle} (leading to potential generalization bounds) but do not attain the proper approximation scaling (as was checked empirically).

Note  that all relaxations   that end up being Lipschitz-continuous functions of $z$, will have at least the same scaling than the set of these functions. The Rademacher complexity of such functions is well-known, that is $1/\sqrt{n}$ for $d=1$, $\sqrt{ \frac{\log n}{n}}$ for $d=2$ and $n^{-1/d}$ for larger $d$~\citep{luxburg2004distance}. Unfortunately, the decay in $n$ is too slow to preserve generalization bounds (which would require a scaling in $1/\sqrt{n}$).

\paragraph{$d$-dimensional relaxation.}
We denote $u_i = (v^\top z_i)_+ = \frac{1}{2} v^\top z_i + \frac{1}{2}|v^\top z_i|$. We may then use 
$
2 u_i - v^\top z_i  = |v^\top z_i|
$
and, for $\| v\|_2=1$, $\| vv^\top z_i\|_2 =  |v^\top z_i| = \sqrt{ z_i ^\top vv^\top z_i}$. By denoting $V = vv^\top$, the constraint that $u_i = (v^\top z_i)_+ = \frac{1}{2} v^\top z_i + \frac{1}{2}|v^\top z_i|$ is equivalent to
$$
\| V z_i \|_2 \leqslant 2 u_i - v^\top z_i \leqslant \sqrt{ z_i^\top V z_i }
\ \mbox{ and } \ V \succcurlyeq 0, \ \tr V = 1,  \ {\rm rank}(V) = 1.
$$
We obtain a convex relaxation when removing the rank constraint, that is
$$
\sup_{ V \succcurlyeq 0, \ \tr V = 1, \ u \in \rb^n  } u^\top y
\ \mbox{ such that } \ \forall i \in \{1,\dots,n\}, \ \| V z_i \|_2 \leqslant 2 u_i - v^\top z_i \leqslant \sqrt{ z_i^\top V z_i }.
$$

  \paragraph{$(n+d)$-dimensional relaxation.}
We may go further by also considering quadratic forms in $u \in \rb^n$ defined above. Indeed, we have:
$$
(2 u_i - v^\top z_i )(2 u_j - v^\top z_j ) = |v^\top z_i| \cdot | v^\top z_j|
= | v^\top z_i z_j^\top v | 
=  |\tr V z_i z_j^\top |,  $$
which leads to a convex program in $U = uu^\top$, $V = vv^\top  $ and $J = uv^\top$, that is a semidefinite program with $d+n$ dimensions, with the constraints
$$
4U_{ij} + x_j^\top V z_i - 2\delta_i^\top J z_j - 2 \delta_j^\top J z_i
\geqslant  |\tr V z_i z_j^\top |,
$$
and the usual semi-definite contraints $\displaystyle
\left( \begin{array}{cc}  U & J \\ J^\top & V \end{array}\right) \succcurlyeq 
\left( \begin{array}{c}  u \\ v  \end{array}\right)\left( \begin{array}{c}  u \\ v  \end{array}\right)^\top $, with the additional constraint that $4U_{ii} + z_i^\top V z_i - 4\delta_i^\top J z_i  = \tr V z_i z_i^\top$.

If we add these constraints on top of the ones above, we obtain a tighter relaxation. Note that for this relaxation, we must have $\big[ (2 u_i - v^\top z_i ) - (2 u_j - v^\top z_j ) \big]$ less than a constant times $\| z_i - z_j\|_2$. Hence, the result mentioned above regarding 
Lipschitz-continuous functions and the scaling of the upper-bound for random $y$ holds (with the dependence on $n$ which is not good enough to preserve the generalization bounds with a polynomial-time algorithm).

\subsection{Relaxation of sign vectors}
By introducing a sign  vector $s \in \rb^n$ such that $s_i \in\{-1,1\}$ and $s_i v^\top x_i  = |v^\top x_i|$, we have the following relaxation with $S = ss^\top$, $V = vv^\top$ and $J = sv^\top$:
\BIT
\item[--] Usual semi-definite constraint: $\displaystyle
\left( \begin{array}{cc}  S & J \\ J^\top & V \end{array}\right) \succcurlyeq 
\left( \begin{array}{c}  s \\ v  \end{array}\right)\left( \begin{array}{c}  s \\ v  \end{array}\right)^\top $,
\item[--] Unit/trace constraints: $\diag(S)=1$ and $\tr V = 1$,
\item[--] Sign constraint: $\delta_i^\top J x_i \geqslant \max_{j \neq i} | \delta_j^\top J x_i |$.
\item[--] Additional constraint: $(x_i^\top V x_i)^{1/2} \leqslant \delta_i^\top J x_i$.
\EIT
We then need to maximize   $
\frac{1}{2n} \sum_{i=1}^n y_i \delta_i^\top J x_i + \frac{1}{2n} \sum_{i=1}^n y_i v^\top x_i
$, which leads to a semidefinte program. Again empirically, it did not lead to the correct scaling as a function of $n$ for random Gaussian vectors $y \in \rb^n$.

\section{Conclusion}

In this paper, we have provided a detailed analysis of the generalization properties of convex neural networks with positively homogenous non-decreasing activation functions. Our main new result is the adaptivity of the method to underlying linear structures such as the dependence on a low-dimensional subspace, a setting which includes non-linear variable selection in presence of potentially many input variables.

All our current results apply to estimators for which no polynomial-time algorithm is known to exist and we have proposed sufficient conditions under which convex relaxations could lead to the same bounds, leaving open the existence or non-existence of such algorithms. Interestingly, these problems have simple geometric interpretations, either as binary linear classification, or computing the Haussdorff distance between two zonotopes.

In this work, we have considered a single real-valued output; the functional analysis framework readily extends to outputs in a finite-dimensional vector-space where vector-valued measures could be used, and then apply to multi-task or multi-class problems. However, the extension to multiple hidden layers does not appear straightforward as the units of the last hidden layers share the weights of the first hidden layers, which should require a new functional analysis framework.

\bibliography{relu}

\appendix
 
 \section{Reproducing kernel Hilbert spaces for $\ell_2$-norm penalization}

 \label{app:rkhs}

  In this section, we consider a Borel probability measure $\tau$ on the compact space $\V$, and functions $\varphi_v: \X \to \rb$ such that the functions $v \mapsto \varphi_v(x)$ are measurable for all $x \in \X$.  We study the set $\F_2$  of functions $f$ such that there exists a squared-integrable function $p:\X\to \rb$ with $f(x) = \int_\V p(v) \varphi_v(x) d\tau(v)$ for all $x \in \mathcal{X}$. For $f \in \F_2$, we define
 $\gamma_2^2(f) $ as the infimum of $\int_\V p(v)^2   d\tau(v)$ over all decompositions of $f$. We now show that $\F_2$ is an RKHS with kernel $k(x,y) = \int_\V \varphi_v(x) \varphi_v(y) d\tau(v)$.
 
We follow the proof of \citet[Section~4.1]{berlinet2004reproducing} and extend it to integrals rather than finite sums. We consider the linear mapping $T: L_2( d\tau) \to \F_2$ defined by
$(T p) (x) = \int_{\V} p(v) \varphi_v(x) d\tau(v)$, with null space $\mathcal{K}$. When restricted to the orthogonal complement $\mathcal{K}^\perp$, we obtain a bijection $U$ from $\mathcal{K}^\perp$ to $\F_2$. We then define a dot-product on $\F_2$ as $
\langle f,g\rangle = \int_{\V} (U^{-1}f)(v)(U^{-1}g)(v) d \tau(v)$. 

We first show that this defines an RKHS with kernel $k$ defined above. For this, we trivially have $k(\cdot,y) \in \F_2$ for all $y \in \X$. Moreover, for any $y \in \X$, we have with 
$p = U^{-1}k(\cdot,y) \in \mathcal{K}^\perp$ and $q: v \mapsto \varphi_v(y)$, $p  - q \in \mathcal{K}$, which implies that 
$
\langle f, k(\cdot,y)\rangle = \int_{\V} (U^{-1}f)(v)p(v) d \tau(v)
= \int_{\V} (U^{-1}f)(v)q(v) d \tau(v) = \int_{\V} (U^{-1}f)(v)\varphi_v(y) d \tau(v) = T ( U^{-1} f)(y) = f(y)$, hence the reproducing property is satisfied. Thus, $\F_2$ is an RKHS.

We now need to show that the RKHS norm which we have defined is actually $\gamma_2$. For any $f \in \F_2$ such that $f = Tp$, for $p \in L_2(d \tau)$, we have $p = U^{-1} f + q$, where $q \in \mathcal{K}$. Thus,
$\int_\V p(v)^2 d\tau(v) = 
 \|p  \|^2_{L_2(d\tau)} = \| U^{-1} f \|^2_{L_2(d\tau)} + \| q\|^2_{L_2(d\tau)} = \| f\|^2 + \| q\|^2_{L_2(d\tau)}$. This implies that 
 $\int_\V p(v)^2 d\tau(v) \geqslant \| f\|^2$ with equality if and only if $q=0$. This shows that $\gamma_2(f) = \| f\|$.

\section{Approximate conditional gradient with multiplicative oracle}

\label{app:cg}

In this section, we wish to minimize a smooth convex functional $J(h)$ on for $h$ in  a Hilbert-space over a norm ball $\{ \gamma(h) \leqslant \delta\}$. A multiplicative approximate oracle outputs for any $g \in \rb^n$, $\hat{h}$ such that $\gamma(\hat{h})=1$, and
$$
\langle \hat{h}, g \rangle \leqslant \max_{\gamma(h) \leqslant 1} \langle h, g \rangle \leqslant \kappa \, \langle \hat{h}, g \rangle,
$$
for a fixed $\kappa \geqslant 1$. We now propose a modification of the conditional gradient algorithm that converges to a certain $h$ such that $\gamma(h) \leqslant \delta$ and for which
$\inf_{\gamma(h) \leqslant  \delta}  J(h)  \leqslant J(\hat{h} ) \leqslant \inf_{\gamma(h) \leqslant   \delta / \kappa}J(h) $.

We assume the smoothness of the function $J$ with respect to the norm $\gamma$, that is, for a certain $L>0$, for all $h,h'$ such that $\gamma(h) \leqslant \delta$, then
\BEQ
\label{eq:LCG}
J(h') \leqslant J(h) + \langle J'(h) , h' - h \rangle + \frac{L}{2} \gamma ( h - h ')^2.
\EEQ
We  consider the following recursion
\BEAS
\hat{h}_t & =  & - \delta  \times  \mbox{ output of approximate oracle at }   - \! J'(h_t) \\
h_{t+1} &  \in & \arg\min_{ \rho \in [0,1] }    J(  (1-\rho)  h_t + \rho     \hat{h}_t  ).
\EEAS
In the previous recursion, one may replace the minimization of $J$ on the segment $[h_t, \hat{h}_t]$ with the  minimization of its upper-bound of \eq{LCG} taken at $h=h_t$. From the recursion, all iterates are in the $\gamma$-ball of radius $\delta$.
Following the traditional convergence proof for the conditional gradient method~\citep{dunn1978conditional,jaggi}, we have, for any $\rho$ in $[0,1]$:
\BEAS
J(h_{t+1}) & \leqslant & J(h_t) - \rho \langle J'(h_t),  h_t -     \hat{h}_t   \rangle +  {2 L \rho^2 \delta^2   }  \\
& = & J(h_t)  - \rho J'(h_t) ^\top  h_t  +    {  \kappa }  \langle J'(h_t),  \frac{\hat{h}_t }{\kappa}  \rangle+ 2 L \rho^2 \delta^2\\
& \leqslant & J(h_t)  - \rho J'(h_t) ^\top  h_t -   { } \max_{\gamma(h) \leqslant \delta/ \kappa} \big\{ - \langle J'(h_t), h \rangle \big\} + 2 L \rho^2 \delta^2.
\EEAS
If we take $h_\ast $ the minimizer of $J$ on $\{\gamma(h) \leqslant \delta / \kappa\}$, we get:
\BEAS
J(h_{t+1}) & \leqslant & J(h_t) - \rho \langle J'(h_t) , h_t -h_\ast \rangle +   {2L \rho^2 \delta^2   }{}.
\EEAS
Then, by
using $J(h_t) \geqslant J(h_\ast) + \langle J'(h_t), h_\ast - h_t\rangle$, we get:
$$
J(h_{t+1}) - J(h_\ast) \leqslant ( 1- \rho) \big[ J(h_{t}) - J(h_\ast) \big]  + {2L \rho^2 \delta^2   }{}.
$$
This is valid for any $\rho \in [0,1]$. If $J(h_{t}) - J(h_\ast) \leqslant 0$ for some $t$, then by taking $\rho=0$ it remains the same of all greater $t$. Therefore, up to (the potentially never happening) point where $J(h_{t}) - J(h_\ast) \leqslant 0$, we can apply the regular proof of the conditional gradien to obtain: 
 $J(h_t) \leqslant \inf_{\gamma(h) \leqslant   \delta / \kappa}J(h)  + \frac{4 L \rho^2 \delta^2}{t}$, which leads to the desired result.
 Note that a similar reasoning may be used for $\rho = 2/(t+1)$.

 \section{Proofs for the $2$-dimensional sphere ($d=1$)}
\label{app:proof1}

In this section, we consider only the case $d=1$, where the sphere $\Sb^d$ is isomorphic to $[0,2\pi]$ (with  periodic boundary conditions). We may then compute the norm $\gamma_2$ in closed form. Indeed,
if we can decompose $g$ as $g(\theta) =\frac{1}{2\pi} \int_{0}^{2 \pi} {p}(\varphi) \sigma (  \cos(\varphi- \theta) ) d\varphi$,
 then  the decomposition of $g$ into the $k$-th frequency elements (the combination of the two $k$-th elements of the Fourier series) is equal to, for $\sigma(u) = (u)_+^\alpha$, and for $k>0$:
\BEAS
\!\!\!\! g_k(\theta) & \!\!=\!\! &  \frac{1}{\pi} \int_{0 }^{2 \pi} g(  \eta) \cos k(\theta - \eta)  d\eta \\
 & \!\!=\!\! &  \frac{1}{\pi} \int_{0 }^{2 \pi}
 \frac{1}{2\pi}  \bigg( \int_{0 }^{2 \pi}  {p}( \varphi) \sigma( \cos (\eta - \varphi) )
  d \varphi \bigg) \cos k(\theta - \eta)  d\eta \mbox{ through the decomposition of } g,\\
 & \!\!=\!\! &  \frac{1}{2\pi^2} 
 \int_{0 }^{2 \pi}  {p}( \varphi)
\bigg(  \int_{0 }^{2 \pi}  \sigma( \cos (\eta - \varphi) ) \cos k(\theta - \eta)  d\eta \bigg) d \varphi\\
 & \!\!=\!\! &  \frac{1}{2\pi^2} 
 \int_{0 }^{2 \pi}  {p}( \varphi)
\bigg(  \int_{0 }^{2 \pi}  \sigma( \cos \eta   ) \cos k(\theta - \varphi - \eta)  d\eta \bigg) d \varphi
\mbox{ by a change of variable},\\
 & \!\!=\!\! &  \frac{1}{2\pi^2} 
 \int_{0 }^{2 \pi}  {p}(  \varphi)
\bigg(  \cos k(\theta - \varphi ) \int_{0 }^{2 \pi}  \sigma( \cos   \eta   )  \cos k \eta \, d\eta \\
& & \hspace*{3cm}
+ \sin k(\theta - \varphi ) \int_{0 }^{2 \pi}  \sigma( \cos   \eta   )  \sin k \eta\,    d\eta
 \bigg) d \varphi \mbox{ by expanding the cosine},\\
 & \!\!=\!\! &  \bigg( \frac{1}{2\pi} 
 \int_{0 }^{2 \pi}  \sigma( \cos   \eta   )  \cos k \eta \,  d\eta
\bigg)
 \bigg(  \frac{1}{\pi} 
 \int_{0 }^{2 \pi}  {p}(  \varphi)
  \cos k(\theta - \varphi )   \bigg)  + 0 \mbox{ by a parity argument},\\
  & \!\!=\!\! &   \lambda_k {p}_k(\theta) \mbox{ with } \lambda_k =  \frac{1}{2\pi} 
 \int_{0 }^{2 \pi}  \sigma( \cos   \eta   )  \cos k \eta \, d\eta.
  \EEAS
  For $k=0$, the same equality holds (except that the two coefficients $g_0$ and $p_0$ are divided by $2\pi$ except of $\pi$).
  
    Thus we may express $\|p \|_{L_2(\Sb^d)}^2$ as
\BEAS
\|p \|_{L_2(\Sb^d)}^2    & = &  \sum_{k \geqslant 0}  \|p_k \|_{L_2(\Sb^d)}^2   =  \sum_{\lambda_k \neq 0 }  \|p_k \|_{L_2(\Sb^d)}^2 + \sum_{\lambda_k = 0 }  \|p_k \|_{L_2(\Sb^d)}^2   
\\
& = &  \sum_{\lambda_k \neq 0 } \frac{1}{\lambda_k^2} \|g_k \|_{L_2(\Sb^d)}^2 + \sum_{\lambda_k = 0 }  \|p_k \|_{L_2(\Sb^d)}^2   .
  \EEAS
  If we minimize over $p$, we thus need to have $\|p_k \|_{L_2(\Sb^d)}^2   =0$ for $\lambda_k=0$, and we get 
  \BEQ
  \gamma_2(g)^2 = \sum_{\lambda_k \neq 0 } \frac{1}{\lambda_k^2} \|g_k \|_{L_2(\Sb^d)}^2.
  \EEQ
We thus simply need to compute $\lambda_k$ and its decay for all values of $\alpha$, and then relate them to the smoothness properties of $g$, which is standard for Fourier series.

\subsection{Computing $\lambda_k$}

We now detail  the computation of  $\lambda_k = \frac{1}{2\pi} 
 \int_{0 }^{2 \pi}  \sigma( \cos   \eta   )  \cos k \eta \, d\eta$ for the different functions $\sigma = (\cdot)_+^\alpha$. 
We have for $\alpha = 0$:
 \BEAS
\frac{1}{2\pi} 
 \int_{0 }^{2 \pi}  1_{\cos   \eta  \geqslant 0}  \cos k \eta \, d\eta
 & = & \frac{1}{2\pi} 
 \int_{-\pi/2 }^{\pi/2}  \cos k \eta  \, d\eta =  \frac{1}{ \pi k } 
\sin \frac{k \pi }{2} \mbox{ if } k \neq 0.
\EEAS
For $k=0$ it is equal to $\frac{1}{2}$.  It is equal to zero for all other even $k$, and different from zero for all odd~$k$, with $\lambda_k$ going to zero as $1/k$.

 We have for $\alpha = 1$:
\BEAS
\frac{1}{2\pi} 
 \int_{0 }^{2 \pi}  (\cos   \eta   )_+  \cos k \eta \, d\eta
 & = & \frac{1}{2\pi} 
 \int_{-\pi/2 }^{\pi/2}  \cos   \eta     \cos k \eta  \, d\eta
\\
& = & \frac{1}{2\pi} 
 \int_{-\pi/2 }^{\pi/2} \big[ \frac{1}{2} \cos   (k+1) \eta + \frac{1}{2} \cos   (k-1) \eta    \big]  \, d\eta\\
 & = & \frac{1}{4\pi} \bigg(
 \frac{2}{k+1} \sin (k+1) \frac{\pi}{2} 
 + \frac{2}{k-1} \sin (k-1) \frac{\pi}{2} 
 \bigg)
\\
 & = & \frac{\cos \frac{k \pi}{2} }{2\pi} \bigg(
 \frac{1}{k+1} 
 - \frac{1}{k-1} 
 \bigg) =  \frac{- \cos \frac{k \pi}{2} }{ \pi( k^2 - 1 )}  \mbox{ for } k \neq 1.\EEAS
For $k=1$, it is equal to $1/4$. It is equal to zero for all other odd $k$, and different from zero for all even $k$, with $\lambda_k$ going to zero as $1/k^2$.

For $\alpha=2$, we have:
\BEAS
\frac{1}{2\pi} 
 \int_{0 }^{2 \pi}  (\cos   \eta   )_+^2  \cos k \eta \, d\eta
 & = & \frac{1}{2\pi} 
 \int_{-\pi/2 }^{\pi/2} ( \cos   \eta)^2     \cos k \eta  \, d\eta
=  \frac{1}{2\pi} 
 \int_{-\pi/2 }^{\pi/2} \frac{1 + \cos 2\eta}{2}     \cos k \eta  \, d\eta
\\
& = & \frac{1}{2\pi} 
 \int_{-\pi/2 }^{\pi/2} \big[ \frac{1}{2} \cos   k  \eta
 + \frac{1}{4} \cos   (k+2) \eta + \frac{1}{4} \cos   (k-2) \eta    \big]  \, d\eta\\
 & = & \frac{1}{4\pi} \bigg(
 \frac{2}{k} \sin k \frac{\pi}{2} +
 \frac{1}{k+2} \sin (k+2) \frac{\pi}{2} 
 + \frac{1}{k-2} \sin (k-2) \frac{\pi}{2} 
 \bigg)\\
  & = & \frac{\sin ( k \frac{\pi}{2}) }{4\pi} \bigg(
 \frac{2}{k}  -
 \frac{1}{k+2}  
 - \frac{1}{k-2}  
 \bigg)
\\
  & = & \frac{\sin ( k \frac{\pi}{2}) }{4\pi} \bigg(
  \frac{2k^2 - 8 - k^2 + 2k - k^2 - 2k}{k(k^2-4)} \bigg)
\\
  & = & \frac{- 8 \sin ( k \frac{\pi}{2}) }{4\pi k(k^2-4) }   \mbox{ for } k  \notin \{0, 2\} .
   \EEAS
For $k=0$, it is equal to $1/4$, and for $k=2$, it is equal to $1/8$.   It is equal to zero for all other even~$k$, and different from zero for all odd $k$, with $\lambda_k$ going to zero as $1/k^3$.

  The general case for $\alpha \geqslant 2$ will be shown for for all $d$ in Appendix~\ref{app:harmonics}: for all $\alpha \in \mathbb{N}$, $\lambda_k$ is different from zero for $k$ having the opposite parity of $\alpha$, with a decay as $1/k^{\alpha+1}$. All values from $k=0$ to $\alpha$ are also different from zero. All larger values with the same parity as $\alpha$ are equal to zero.

\subsection{Proof of  Prop.~\ref{prop:finites-sphere} for $d=1$}
 \label{app:proof11}
We only consider the proof for $d=1$. For the proof for general $d$, see Appendix~\ref{app:proofd1}.

Given the zero values of $\lambda_k$ given above,   if $g$ has the opposite parity than $\alpha$ (that is, is even when $\alpha$ is odd, and vice-versa), then we may define ${p}$ through its Fourier series, which is obtained by multiplying the one of $g$ by a strictly positive sequence  growing as $k^{\alpha+1}$.

Thus, if $g$ is such that its $(\alpha+1)$-th order derivative is squared-integrable, then ${p}$ defined above is squared-integrable, that is, $g \in \G_2$. Moreover, if all derivatives of order less than $(\alpha+1)$ are bounded by $\eta$, ${p}$ is squared-integrable and $\| p\|_{L_2(\Sb^d)}^2$ is upper-bounded by a constant times $\eta^2$, i.e., $\gamma_2(g)^2 \leqslant C(\alpha)^2\eta^2$.

Note that we could relax the assumption that $g$ is even (resp.~odd) by adding all trigonometric polynomials 
 of order less than $\alpha$.

 \subsection{Proof of Prop.~\ref{prop:approx-sphere}  for $d=1$}
  \label{app:appsphere2}
Again, we only consider the proof for $d=1$. For the proof for general $d$, see Appendix~\ref{app:proofd2}. 

Without loss of generality, we assume that $\eta=1$. For $d=1$, we essentially want to approximate a Lipschitz-continuous function by a function which is
  $(\alpha+1)$-times differentiable.
  
  For $\alpha=0$, then the function $g$ is already in $\G_2$ with a norm less than one, because Lipschitz-continuous functions are almost everywhere differentiable with bounded derivative~\citep{adams2003sobolev}. We thus now consider $\alpha>0$.
  
  Given $\lambda_k$ defined above and $r \in (0,1)$, we define $\hat{{p}}$ through
$$\hat{p}_k(\theta) = \sum_{k, \lambda_k \neq 0} \lambda_k^{-1} g_k(\theta) r^k.$$
Our goal is to show that for $r$ chosen close enough to $1$, then the function $\hat{g}$ defined from  $\hat{p}$ has small enough norm $\gamma_2(\hat{g}) \leqslant \| \hat{p}\|_{L_2(\Sb^d)}$, and is close to $g$.

\paragraph{Computation of norm.}
We have
\BEAS
\| \hat{p} \|^2_{L_2(\Sb^d)}   
& = & \sum_{k, \lambda_k \neq 0} \lambda_k^{-2} r^{2k}  \| g_k \|^2_{L_2(\Sb^d)} 
.
\EEAS
Since $g$ is $1$-Lipschitz-continuous with constant $1$, then it has a squared-integrable derivative $f=g'$ with norm less than 1~\citep{adams2003sobolev}, so that
$$
\| f \|^2_{L_2(\Sb^d)}  = 
\sum_{k \geqslant 0} \| f_k \|^2_{L_2(\Sb^d)} \leqslant 1.$$

This implies that using $\lambda_k^{-1} = O(k^{\alpha+1})$:
  $$
  \| \hat{{p}} \|_{L_2(\Sb^d)}^2 \leqslant \lambda_0^{-2} \| g_0 \|^2_{L_2(\Sb^d)} + \| g'  \|_{L_2(\Sb^d)}^2
  \max_{ k \geqslant 1, \lambda_k \neq 0} r^{2k} \lambda_k^{-2} k^{-2}
  \leqslant  C + C \| g'  \|_{L_2(\Sb^d)}^2
  \max_{ k \geqslant 1} r^{2k} k^{2 \alpha},
  $$
  because $\| g_0 \|^2_{L_2(\Sb^d)}$ and $\| f \|^2_{L_2(\Sb^d)}$ are bounded by 1.

  We may now compute the derivative of $k \mapsto  r^{2k} k^{2 \alpha}$ with respect to $k$ (now considered a real number),
  that is $2\alpha k^{2\alpha -1}  r^{2k}
  + k^{2 \alpha} r^{2k} 2 \log r$, which is equal to zero for $\frac{ \alpha}{k}  = \log \frac{1}{r}$, that is
  $k = \frac{\alpha}{ \log \frac{1}{r}}$, the maximum being then
  $e^{-2\alpha} \big( \frac{\alpha}{ \log \frac{1}{r}} \big) ^{2 \alpha} = O ( ( 1- r)^{-2\alpha})$, by using the concavity of the logarithm.
    Thus 
$ \displaystyle   \| \hat{{p}} \|_{L_2(\Sb^d)}  \leqslant C ( 1- r)^{-\alpha}.$
  This defines $\hat{g}$ with $\gamma(\hat{g}) \leqslant C ( 1- r)^{-\alpha}$.

  \paragraph{Computing distance between $\hat{g}$ and $g$.}
  We have:
  \BEAS
 \hat{g}(\theta)
  & = & \sum_{k\geqslant 0} g_k(\theta) r^k =   \sum_{k>  0}  \frac{1}{\pi} \int_{0 }^{2 \pi} g(  \eta) r^k \cos k(\theta - \eta)  d\eta  
  + \frac{1}{2 \pi} \int_0^{2\pi} g(\eta) d\eta\\
  & = & \frac{1}{\pi} \int_0^{2\pi}    \bigg(
   \sum_{k\geqslant 0} r^k \cos k(\theta - \eta) 
  \bigg)  g(  \eta) d\eta - \frac{1}{2 \pi} \int_0^{2\pi} g(\eta) d\eta \\
   & = &\frac{1}{\pi}  \int_0^{2\pi}    {\rm Real} \bigg( \frac{1}{1 - r e^{i (\theta - \eta) }}
  \bigg)  g(  \eta) d\eta  - \frac{1}{2 \pi} \int_0^{2\pi} g(\eta) d\eta
 \\
  & = &\frac{1}{\pi}  \int_0^{2\pi}    \bigg( \frac{1 - r \cos  (\theta - \eta) }{
  (1 - r \cos  (\theta - \eta) )^2 + r^2 (\sin  (\theta - \eta))^2
   }
  \bigg)  g(  \eta) d\eta  - \frac{1}{2 \pi} \int_0^{2\pi} g(\eta) d\eta
\\
  & = &\frac{1}{\pi}  \int_0^{2\pi}    \bigg( \frac{1 - r \cos  (\theta - \eta) }{
 1 + r^2 - 2 r  \cos  (\theta - \eta) 
   }
  \bigg)  g(  \eta) d\eta  - \frac{1}{2 \pi} \int_0^{2\pi} g(\eta) d\eta
\\
 & = &\frac{1}{2\pi}  \int_0^{2\pi}    \bigg( \frac{ 1 - r^2 +  1 + r^2  - 2 r \cos  (\theta - \eta) }{
 1 + r^2 - 2 r  \cos  (\theta - \eta) 
   }
  \bigg)  g(  \eta) d\eta  - \frac{1}{2 \pi} \int_0^{2\pi} g(\eta) d\eta
\\
 & = &\frac{1}{2\pi}  \int_0^{2\pi}    \bigg( \frac{ 1 - r^2  }{
 1 + r^2 - 2 r  \cos  (\theta - \eta) 
   }
  \bigg)  g(  \eta) d\eta.   \EEAS
  We have, for any $\theta \in [0,2\pi]$
  \BEAS
 \big|  \hat{g}(\theta) - g(\theta) \big| 
  & = &\bigg| \frac{1}{2 \pi}  \int_0^{2\pi}    \bigg( \frac{1 - r^2 }{
 1 + r^2 - 2 r  \cos  (\theta - \eta) 
   }
  \bigg)  \big[ g(  \eta)  - g(\theta) \big] d\eta  \bigg| \\
 & \leqslant & \frac{1}{2 \pi}  \int_0^{2\pi}    \bigg( \frac{1 - r^2 }{
 1 + r^2 - 2 r  \cos  (\theta - \eta) 
   }
  \bigg)  \big| g(  \eta)  - g(\theta) \big| d\eta  \  \\ 
  & = & \frac{1}{2 \pi}  \int_{\pi}^{\pi}    \bigg( \frac{1 - r^2 }{
 1 + r^2 - 2 r  \cos  \eta
   }
  \bigg)  \big| g(  \theta )  - g(\theta + \eta) \big| d\eta  \  \mbox{ by periodicity}, \\
    & = & \frac{1}{ \pi}  \int_{\pi/2}^{\pi/2}    \bigg( \frac{1 - r^2 }{
 1 + r^2 - 2 r  \cos  \eta
   }
  \bigg)  \big| g(  \theta )  - g(\theta + \eta) \big| d\eta  \  \mbox{ by parity of } g, \\
  & \leqslant & \frac{1}{  \pi} \int_{\pi/2}^{\pi/2}     \bigg( \frac{1 - r^2 }{
 1 + r^2 - 2 r  \cos \eta
   }
  \bigg) \sqrt{2}| \sin \eta |   d\eta    \\  
  & & \hspace*{1cm} \mbox{ because the distance on the sphere is bounded by the sine}, \\
  & \leqslant & \frac{2}{  \pi}  \int_0^{\pi}    \bigg( \frac{1 - r^2 }{
 1 + r^2 - 2 r  \cos   \eta 
   }
  \bigg)    \sin \eta  \,   d\eta    \\  
& = & \frac{1}{  \pi}  \int_0^{1}    \bigg( \frac{1 - r^2 }{
 1 + r^2 - 2 r  t
   }
  \bigg)    dt    \mbox{ by the change of variable } t = \cos \theta , \\  
&\leqslant  & C ( 1 - r) \int_0^{1}    \bigg( \frac{1}{
 1 + r^2 - 2 r  t
   }
  \bigg)    dt    \\  
&=  & C ( 1 - r)  \bigg[    \frac{-1}{2r} \log ( 
 1 + r^2 - 2 r  t)
  \bigg]_0^1  =   C ( 1 - r)       \frac{ 1}{2r} \log \frac{1+r^2}{(1-r)^2}     .
\EEAS
It can be easily checked that for any $r \in (1/2,1)$, the last function is less than a constant times
$ \frac{5}{2} C ( 1 - r ) \log \frac{1}{1-r}  $. 
We thus get for $\delta  $ large enough, by taking $r =   1 - ( C/ \delta)^{1/\alpha} \in (1/2,1)$,  an error of
$$
( C/ \delta)^{1/\alpha} \log ( C/ \delta)^{-1/\alpha}  = O(  \delta^{-1/\alpha} \log \delta ).
$$ This  leads to the desired result.

 \section{Approximations on the $d$-dimensional sphere}

In this appendix, we first review tools from spherical harmonic analysis, before proving the two main propositions regarding the approximation properties of the Hilbert space  $\G_2$. Using spherical harmonics in our set-up is natural and is common in the analysis of ridge functions~\citep{petrushev1998approximation} and zonotopes~\citep{bourgain1988projection}.

 \subsection{Review of spherical harmonics theory}
 \label{app:review}
 
 In this section, we review relevant concepts from spherical harmonics. See~\citet{2012arXiv1205.3548F,atkinson2012spherical} for more details. Spherical harmonics may be seen as extension of Fourier series to spheres in dimensions more than 2 (i.e., with our convention $d \geqslant 1$).

For $d \geqslant 1$, we consider the sphere $\Sb^{d} = \{ x \in \rb^{d+1}, \ \| x\|_2 = 1\} \subset \rb^{d+1}$, as well as its normalized rotation-invariant measure $\tau_d$ (with mass $1$). We denote by $\omega_d = \frac{2 \pi^{(d+1)/2}}{ \Gamma((d+1)/2) }$ the surface area of the sphere~$\Sb^d $.

\paragraph{Definition and links with Laplace-Beltrami operator.}
For any $k \geqslant 1$ (for $k=0$, the constant function is the corresponding basis element), there is an orthonormal basis of spherical harmonics, $Y_{kj}:\Sb^d  \to \rb$, $ 1\leqslant j
\leqslant N(d,k) = \frac{2k+d-1}{k} { k+d -2 \choose d-1}$. They are such
$
\langle
Y_{ki}, Y_{si}
\rangle_{\Sb^d} =  \int_{\Sb^d } Y_{ki}(x) Y_{sj} d\tau_d(x) = \delta_{ij}\delta_{sk}$. 

Each of these harmonics may be obtained from homogeneous polynomials in $\rb^d$ with an Euclidean Laplacian equal to zero, that is, if we define a function $H_k(y) = Y_{ki}( y / \| y\|_2 ) \| y\|_2^k$ for $y \in \rb^{d+1}$, then $H_k$ is a homogeneous polynomial of degree $k$ with zero Laplacian. From the relationship between the Laplacian in $\rb^{d+1}$ and the Laplace-Beltrami operator $\Delta$ on $\Sb^d$, $Y_{ki}$ is an eigenfunction of $\Delta$ with eigenvalue $-k(k+d-1)$. Like in Euclidean spaces, the Laplace-Beltrami operator may be used to characterize differentiability of functions defined on the sphere~\citep{2012arXiv1205.3548F,atkinson2012spherical}.

\paragraph{Legendre polynomials.}
We have the addition formula
$$
\sum_{j=1}^{N(d,k)} Y_{kj}(x) Y_{kj}(y) =  N(d,k)  P_k(x^\top y),
$$
where $P_k$ is a Legendre polynomial of degree $k$ and dimension $d+1$, defined as (Rodrigues' formula):
$$
P_k(t) = (-1/2)^k \frac{ \Gamma(d/2)}{\Gamma( k + d /2)} ( 1- t^2)^{(2-d)/2} \Big(
\frac{d}{dt}
\Big)^k ( 1 - t^2)^{k+(d-2)/2}.
$$
They are also referred to as Gegenbauer polynomials.
For $d=1$, $P_k$ is the $k$-th Chebyshev polynomial, such that $P_k(\cos \theta) = \cos (k \theta)$ for all $\theta$ (and we thus recover the Fourier series framework of Appendix~\ref{app:proof1}).
For $d=2$, $P_k$ is the usual Legendre polynomial.  

The polynomial $P_k$ is even (resp.~odd) when $k$ is even (resp.~odd), and we have
$$
\int_{-1}^1  P_k(t) P_j(k)  (1-t^2)^{(d-2)/2} dt =  \delta_{jk} \frac{\omega_d}{\omega_{d-1}} \frac{1}{N(d,k)}.
$$

 For small $k$, we have    $P_0(t) = 1$, $P_1(t) = t$, and $P_2(t) = \frac{ (d+1) t^2 - 1}{d}$.

The Hecke-Funk formula leads to, for any linear combination $Y_k$ of $Y_{kj}$, $j \in \{1,\dots,N(d,k)\}$:
$$
\int_{\Sb^d } f(x^\top y) Y_k(y) d\tau_d(y) = \frac{\omega_{d-1}}{\omega_d} Y_k(x) \int_{-1}^1 f(t) P_k(t) ( 1- t^2)^{(d-2)/2} dt.
$$

\paragraph{Decomposition of functions in $L_2(\Sb^d)$.}

Any function $g: \Sb^d  \to \rb$, such that $\int_{\Sb^d } g(x) d \tau_d(x) =0$ may be decomposed
as
\BEAS
g(x) & = &  \sum_{k=1}^\infty \sum_{j=1}^{N(d,k)}  \langle Y_{kj}, g \rangle Y_{kj}(x)  =  \sum_{k=1}^\infty \sum_{j=1}^{N(d,k)} \int_{\Sb^d } Y_{kj}(y) Y_{kj}(x) g(y) d\tau_d(y)
\\
& = &  \sum_{k=1}^\infty g_k(x) \mbox{ with } g_k(x) = N(d,k) \int_{\Sb^d }g(y) P_k(x^\top y) d\tau_d(y).
\EEAS
This is the decomposition in harmonics of degree $k$.
Note that
$$
g_1(x) = x^\top \bigg[
d \int_{\Sb^d } y g(y) d\tau_d(y)
\bigg]
$$
is the linear part of $g$ (i.e., if $g(x) = w^\top x$, $g_1=g$). Moreover, if $g$ does not have zero mean, we may define $g_0(x) = \int_{\Sb^d} g(y) d \tau_d(y)$ as the average value of $g$. Since the harmonics of different degrees are orthogonal to each other, we have the Parseval formula:
$$
\| g\|_{\Sb^d}^2 = \sum_{k \geqslant 0} \| g_k\|_{\Sb^d}^2.
$$

\paragraph{Decomposition of functions of one-dimensional projections.}
If $g(x) = \varphi(x^\top v)$ for $v \in \Sb^d$ and $\varphi: [-1,1] \to \rb$, then
\BEAS
g_k(x) & = & N(d,k) \int_{\Sb^d} \varphi(v^\top y) P_k(x^\top y) d \tau (y) \\
& = & N(d,k) \frac{\omega_{d-1}}{\omega_d} P_k(x^\top v) \int_{-1}^1 \varphi(t) P_k(t) ( 1 - t^2)^{(d-2)/2} dt\\
& = & \bigg(
\frac{\omega_{d-1}}{\omega_d} P_k(x^\top v) \int_{-1}^1 \varphi(t) P_k(t) ( 1 - t^2)^{(d-2)/2} dt \bigg)
\sum_{j=1}^{N(d,k)} Y_{kj}(x) Y_{kj}(y),
\EEAS
and
\BEAS
\| g_k \|_{L_2(\Sb^d)}^2
& = & 
\bigg(
\frac{\omega_{d-1}}{\omega_d} P_k(x^\top v) \int_{-1}^1 \varphi(t)P_k(t) ( 1 - t^2)^{(d-2)/2} dt \bigg)^2
\sum_{j=1}^{N(d,k)}   Y_{kj}(y)^2 \\
& = & 
\bigg(
\frac{\omega_{d-1}}{\omega_d} P_k(x^\top v) \int_{-1}^1 \varphi(t)P_k(t) ( 1 - t^2)^{(d-2)/2} dt \bigg)^2
N(d,k) P_k(1)\\
& = & 
\bigg(
\frac{\omega_{d-1}}{\omega_d} P_k(x^\top v) \int_{-1}^1 \varphi(t)P_k(t) ( 1 - t^2)^{(d-2)/2} dt \bigg)^2
N(d,k).
\EEAS

\subsection{Computing the RKHS norm $\gamma_2$}
\label{app:harmonics}
Like for the case $d=1$, we may compute the RKHS norm $\gamma_2$ of a function $g$ in closed form given its decomposition in the basis of spherical harmonics $g = \sum_{k \geqslant 0} g_k$.
If we can decompose $g(x) = \int_{\Sb^d } {p}(w) \sigma( w^\top x ) d \tau_d(w)$ for a certain function ${p}: \Sb^d \to \rb$, then we have, for $k \geqslant 0$:
\BEAS
g_k(x) & = &  N(d,k) \int_{\Sb^d }g(y) P_k(x^\top y) d\tau_d(y) \\
& = &  N(d,k) \int_{\Sb^d } \int_{\Sb^d }   {p}(w) \sigma(w^\top y)  P_k(x^\top y) d\tau_d(y)  d\tau_d(w) \\
& = &  N(d,k) \int_{\Sb^d }  {p}(w) \bigg( \int_{\Sb^d }    \sigma(w^\top y)  P_k(x^\top y) d\tau_d(y)  \bigg) d\tau_d(w) \\
& = &  \frac{\omega_{d-1}}{\omega_d} N(d,k) \int_{\Sb^d }  {p}(w) P_k(x^\top w)  \bigg( \int_{-1}^1   \sigma(t)   P_k(t) (1-t^2)^{(d-2)/2} dt  \bigg) d\tau_d(w)   \\
&  &\hspace*{8cm} \mbox{ using the Hecke-Funk formula,}\\
&= &  \lambda_k {p}_k(x) \mbox{ with } \lambda_k = \frac{\omega_{d-1}}{\omega_d}   \int_{-1}^1   \sigma(t)   P_k(t) (1-t^2)^{(d-2)/2} dt.
 \EEAS


If $k \equiv \alpha  \mbox{ mod. } 2$, then 
$
\lambda_k \propto \frac{1}{2} \int_{-1}^1   t^\alpha  P_k(t) (1-t^2)^{(d-2)/2} dt  = 0
$, for $k > \alpha$ since $P_k$ is orthogonal to all polynomials of degree strictly less than $k$ for that dot-product. Otherwise, $\lambda_k \neq 0$, since $t^\alpha$ may be decomposed as combination with non-zero coefficients of polynomials $P_j$ for $j \equiv \alpha \mbox{ mod. } 2$, $j \leqslant \alpha$. 

We now provide an explicit formula extending the proof technique (for $\alpha=1$) of~\citet{schneider1967problem} and \citet{bourgain1988projection}  to all values of $\alpha$. See also~\citet{mhaskar2006weighted}.

We have, by $\alpha$ successive integration by parts, for $k \geqslant \alpha + 1$:
\BEAS
 & & \int_0^1 t^\alpha 
\Big(
\frac{d}{dt}
\Big)^k ( 1 - t^2)^{k+(d-2)/2} dt \\
& = & (-1)^\alpha \alpha!  \int_0^1  
\Big(
\frac{d}{dt}
\Big)^{k-\alpha} ( 1 - t^2)^{k+(d-2)/2} dt =  -  (-1)^\alpha \alpha!   
\Big(
\frac{d}{dt}
\Big)^{k-\alpha-1} ( 1 - t^2)^{k+(d-2)/2} \bigg|_{t=0}\\
& = &- (-1)^\alpha \alpha!   
\Big(
\frac{d}{dt}
\Big)^{k-\alpha-1} \sum_{j \geqslant 0} {k+(d-2)/2 \choose j} (-1)^j t^{2j}  \bigg|_{t=0}
\mbox{ using the binomial formula},\\
& = & -(-1)^\alpha \alpha!   
\Big(
\frac{d}{dt}
\Big)^{k-\alpha-1}  {k+(d-2)/2 \choose j} (-1)^j t^{2j}  \bigg|_{t=0} \mbox{ for } 2j = k-\alpha-1,
\\
& = & -(-1)^\alpha \alpha!   
  {k+(d-2)/2 \choose j} (-1)^j  (2j)!    \mbox{ for } 2j = k-\alpha-1.
\EEAS
Thus
\BEAS
\lambda_k
& = & -\frac{\omega_{d-1}}{\omega_d}  
 (-1/2)^k \frac{ \Gamma(d/2)}{\Gamma( k + d /2)}
  (-1)^\alpha \alpha!   
  {k+(d-2)/2 \choose j} (-1)^j  (2j)!    \mbox{ for } 2j = k-\alpha-1,
  \\
  & = & -\frac{\omega_{d-1}}{\omega_d}  
 (-1/2)^k \frac{ \Gamma(d/2)}{\Gamma( k + d /2)}
  (-1)^\alpha \alpha!   
  \frac{\Gamma(k+\frac{d}{2}) }{\Gamma(j+1) \Gamma( k + \frac{d}{2} - j) }(-1)^j  \Gamma(2j+1) \\
    & = & -\frac{\omega_{d-1}}{\omega_d}  
 (-1/2)^k \frac{ \Gamma(d/2)}{\Gamma( k + d /2)}
  (-1)^\alpha \alpha!   
  \frac{\Gamma(k+\frac{d}{2}) }{\Gamma(\frac{k}{2} - \frac{\alpha}{2} + \frac{1}{2}) \Gamma( \frac{k}{2} + \frac{d}{2} + \frac{\alpha}{2} + \frac{1}{2}) }(-1)^{(k-\alpha-1)/2}  \Gamma(k - \alpha)\\
    & = & \frac{d-1}{2\pi}
 \frac{ \alpha! (-1)^{(k-\alpha-1)/2}  }{2^k}  
  \frac{ \Gamma(d/2) \Gamma(k - \alpha) }{\Gamma(\frac{k}{2} - \frac{\alpha}{2} + \frac{1}{2}) \Gamma( \frac{k}{2} + \frac{d}{2} + \frac{\alpha}{2} + \frac{1}{2}) }     .
\EEAS
By using Stirling formula $\Gamma(x) \approx x^{x-1/2} e^{-x} \sqrt{2\pi}$, we get an equivalent when $k$ 
or $d$ tends to infinity as a constant (that depends on $\alpha$) times 
$$
d^{d/2+1/2} k^{k/2 - \alpha/2 + 1/2} ( k + d)^{-k/2-d/2-\alpha/2}.
$$
Note that all exponential terms cancel out. Moreover, when $k$ tends to infinity and $d$ is considered constant, then we get the equivalent  $k^{-d/2-\alpha-1/2}$, which we need for the following sections.
Finally, when $d$ tends to infinity and $k$ is considered constant, then we get the equivalent  $d^{-\alpha/2 - k/2 + 1/2}$.

We will also need expressions of $\lambda_k$ for $k=0$ and $k=1$.
For $k = 0$, we have:
\BEAS
\int_0^1 t^\alpha 
  ( 1 - t^2)^{d/2-1} dt
& = &  \int_0^1 (1-u)^{\alpha/2} u^{d/2-1} \frac{du}{2\sqrt{1-u}} \mbox{ with } t = \sqrt{1-u},\\
& = &  \frac{1}{2} \int_0^1 (1-u)^{\alpha/2+1/2 -1 } u^{d/2-1}  du=\frac{1}{2} \frac{\Gamma(\alpha/2+1/2 )\Gamma(d/2) }{\Gamma(\alpha/2+1/2+d/2)},
\EEAS
using the normalization factor of the Beta distribution.
This leads to
\BEAS
\lambda_0 & = &  \frac{\omega_{d-1}}{\omega_d} \frac{1}{2} \frac{\Gamma(\alpha/2+1/2 )\Gamma(d/2) }{\Gamma(\alpha/2+1/2+d/2)} = \frac{d-1}{2\pi}
 \frac{1}{2} \frac{\Gamma(\alpha/2+1/2 )\Gamma(d/2) }{\Gamma(\alpha/2+1/2+d/2)} ,
\EEAS
which is equivalent to $d^{1/2 -\alpha/2}$ as $d$ tends to infinity.

Moreover, for $k=1$, we have (for $\alpha>0$):
\BEAS
\int_0^1 t^\alpha 
\Big(
\frac{d}{dt}
\Big) ( 1 - t^2)^{d/2} dt
&\!\!\! = \!\!\!& -\alpha  \int_0^1 t^{\alpha-1}
 ( 1 - t^2)^{d/2} dt
 = -\alpha  \int_0^1 (1-u)^{\alpha/2-1/2}
 u^{d/2} \frac{du}{2\sqrt{1-u}} \\
& \!\!\!= \!\!\!& 
 -\alpha/2  \int_0^1 (1-u)^{\alpha/2-1}
 u^{d/2+1-1} du
 =  - \alpha/2  \frac{\Gamma(\alpha/2) \Gamma(d/2+1)}{\Gamma(\alpha/2+d/2+1)}.
\EEAS
This leads to, for $\alpha>0$:
\BEAS
\lambda_1 & = &  (-1/2) \frac{2}{d} \frac{d-1}{2\pi} (- \alpha/2)  \frac{\Gamma(\alpha/2) \Gamma(d/2+1)}{\Gamma(\alpha/2+d/2+1)} =   \frac{d-1}{d}\frac{\alpha}{4\pi}   \frac{\Gamma(\alpha/2) \Gamma(d/2+1)}{\Gamma(\alpha/2+d/2+1)},
\EEAS
which is equivalent to $d^{  -\alpha/2}$ as $d$ tends to infinity.

Finally, for $\alpha=0$, $\lambda_1 = \frac{d-1}{2d \pi}$.
More generally, we have $|\lambda_k| \sim C(d) k^{-(d-1)/2 - \alpha  - 1}$.

\paragraph{Computing the RKHS norm.}
Given $g$ with the correct parity, then we have
$$
\gamma_2(g)^2 = \sum_{ k \geqslant 0} \lambda_k^{-2} \| g_k \|_{L_2(\Sb^d)}^2.
$$

\subsection{Proof of Prop.~\ref{prop:finites-sphere} for $d>1$}

\label{app:proofd1}

Given the expression of $\lambda_k$ from the section above, the proof is essentially the same than for $d=1$ in Appendix~\ref{app:appsphere2}.
If $g$ is $s$-times differentiable with all derivatives bounded uniformly by $\eta$, then is equal to
$g = \Delta^{s/2} f$ for a certain function $f$ such that $\| f\|_{L_2(\Sb^d)} \leqslant \eta$ (where $\Delta$ is the Laplacian on the sphere)~\citep{2012arXiv1205.3548F,atkinson2012spherical}.

Moreover, since $g$ has the correct parity,  
$$\gamma_2(g)^2 \leqslant \| {p}\|_{L_2(\Sb^d)}^2
\leqslant   \sum_{k \geqslant 1, \lambda_k \neq 0 } \lambda_k^{-2} \| g_k \|_{L_2(\Sb^d)}^2
$$
  Also,  $g_k$ are eigenfunctions of the Laplacian with eigenvalues $k(k+d-1)$. Thus, we have
$$
\| g_k \|_2^2
\leqslant \big\| f_k \big\|_{L_2(\Sb^d )}^2 \frac{1}{ \big[ k(k+d-1) \big]^{s} } \leqslant \big\| f_k \big\|_{L_2(\Sb^d )}^2 / k^{2s},
$$
leading to
$ \displaystyle
\gamma_2(g)^2 \leqslant \max_{k \geqslant 2} \lambda_k^{-2}  k^{-2s} \| f\|_{L_2(\Sb^d )}^2
 \leqslant \max_{k \geqslant 2} k^{d-1 + 2 \alpha + 2}  k^{-2s} \| f\|_{L_2(\Sb^d )}^2
 \leqslant C(d) \eta^2$, 
 if $s \geqslant (d-1)/2 + \alpha + 1$, which is the desired result.

\subsection{Proof of Prop.~\ref{prop:approx-sphere} for $d>1$}
\label{app:proofapp}
\label{app:proofd2}
 Without loss of generality we assume that $\eta=1$, and we follow the same proof as for $d=1$ in Appendix~\ref{app:appsphere2}.
  We have assumed that for all $x, y \in \Sb^d $, $|g(x)-g(y)| \leqslant \eta \| x - y \|_2
 = \eta \sqrt{ 2} \sqrt{  1 - x^\top y}
 $.   
 Given the decomposition in the $k$-th harmonics, with 
 \BEAS
 g_k(x) & = &  N(d,k) \int_{\Sb^d }g(y) P_k(x^\top y) d\tau_d(y),
  \EEAS
  we may now   define, for $r<1$:
 \BEAS
 \hat{{p}}(x) & = & \sum_{k, \lambda_k \neq 0 }   \lambda_k^{-1} r^k 
 g_k(x),
 \EEAS
 which is always defined when $r \in (0,1)$ because the series is absolutely convergent. This defines a function $\hat{g}$ that will have a finite $\gamma_2$-norm and be close to $g$.
 
 \paragraph{Computing the norm.}
Given our assumption regarding the Lipschitz-continuity of $g$, we have $ g  = \Delta^{1/2} f$ with $f \in L_2(\Sb^d)$ with norm less than $1$~\citep{atkinson2012spherical}. Moreover $
 \| g_k \|_{L_2(\Sb^d )}^2 \leqslant C k^2 \| f_k \|_{L_2(\Sb^d )}^2$.
 We have
 \BEAS
 \|  \hat{{p}}\|_{L_2(\Sb^d )}^2 & = & \sum_{k, \lambda_k \neq 0 }   \lambda_k^{-2} r^{2k} \| g_k \|_{L_2(\Sb^d )}^2 \\
& \leqslant
& C(d , \alpha) \max_{k \geqslant 0} k^{d-1 + 2 \alpha }  r^{2k}  \|f \|_{L_2(\Sb^d )}^2 \mbox{ because } 
\lambda_k = \Omega( k^{-d/2-\alpha-1/2}), \\
& \leqslant
& C(d , \alpha) (1-r)^{-d+1-2\alpha} \mbox{ (see Appendix~\ref{app:appsphere2}).}
\EEAS

The function $\hat{{p}}$ thus defines a function $\hat{g} \in \G_1$ by $\hat{g}_k  = \lambda_k {p}_k$, for which $  \gamma_2(g) \leqslant  C(d , \alpha) (1-r)^{(-d+1)/2-\alpha}$. 

\paragraph{Approximation properties.}

We now show that $g$ and $\hat{g}$ are close to each other. Because of the parity of $g$, we have $\hat{g}_k = r^k g_k$.
We have, using Theorem 4.28 from~\citet{2012arXiv1205.3548F}:
\BEAS
\hat{g}(x)
& = & \sum_{k \geqslant 0} r^k = \sum_{k \geqslant 0} r^k  N(d,k) \int_{\Sb^d }g(y) P_k(x^\top y) d\tau_d(y) \\
& = &  \int_{\Sb^d }g(y)  \bigg(  \sum_{ k \geqslant 0 } r^k N(d,k)  P_k(x^\top y)  \bigg) d\tau_d(y) 
\\
& = &  \int_{\Sb^d }g(y)   \frac{ 1- r^2}{(1 + r^2 - 2r (x^\top y))^{(d+1)/2} } d\tau_d(y) .
\EEAS
    
Moreover, following~\citet{bourgain1988projection}, we have:
 \BEAS
  g(x) -  \hat{g}(x) &
  = &   \int_{\Sb^d }\big[ g(x) - g(y) ]   \frac{ 1- r^2}{(1 + r^2 - 2r (x^\top w))^{(d+1)/2} } d\tau_d(y) \\
 &
  = &  2 \int_{\Sb^d , \ y^\top x \geqslant 0 }\big[ g(x) - g(y) ]   \frac{ 1- r^2}{(1 + r^2 - 2r (x^\top w))^{(d+1)/2} } d\tau_d(y) \mbox{ by parity of } g, \\
|   g(x) -  \hat{g}(x) | &
  \leqslant &   \int_{\Sb^d , \ y^\top x \geqslant 0 } \sqrt{2} \sqrt{ 1 - x^\top y}    \frac{ 1- r^2}{(1 + r^2 - 2r (x^\top y))^{(d+1)/2} } d\tau_d(y).
  \EEAS
  As shown by~\citet[Eq.~(2.13)]{bourgain1988projection}, this is less than a constant that depends on $d$ times $ ( 1- r) \log \frac{1}{1-r}$.
 We thus get for $\delta $ large enough, by taking $1-r =  ( C/ \delta)^{1/(\alpha+(d-1)/2)} \in (0,1)$, an error of
$$
( C/ \delta)^{1/(\alpha+(d-1)/2)}  \log ( C/ \delta)^{-{1/(\alpha+(d-1)/2)} } ] = O(  \delta^{{1/(\alpha+(d-1)/2)} } \log \delta ),
$$
which leads to the desired result.

 \subsection{Finding differentiable functions which are not in $\G_2$}
 
 \label{app:tightness}

 In this section, we consider functions on the sphere which have the proper parity with respect to $\alpha$, which are $s$-times differentiable with bounded derivatives, but which are not in $\G_2$. We then provide optimal approximation rates for these functions.
 
 We assume that $s-\alpha$ is even, we consider $g(x)=(w^\top x)_+^s$ for a certain arbitrary $w \in \Sb^d$. As computed at the end of Appendix~\ref{app:review}, we have
 $ \| g_k \|_{L_2(\Sb^d)}^2 = \bigg(
\frac{\omega_{d-1}}{\omega_d} P_k(x^\top v) \int_{-1}^1 \varphi(t)P_k(t) ( 1 - t^2)^{(d-2)/2} dt \bigg)^2
N(d,k)$. Given the computations from Appendix~\ref{app:harmonics}, 
$\bigg(
\frac{\omega_{d-1}}{\omega_d} P_k(x^\top v) \int_{-1}^1 \varphi(t)P_k(t) ( 1 - t^2)^{(d-2)/2} dt \bigg)^2$ goes down to zero as $k^{-d-2s-1}$, while $N(d,k)$ grows as $k^{d-1}$.  In order to use the computation of the RKHS norm derived in Appendix~\ref{app:harmonics}, we need to make sure that $g$ has the proper parity. This can de done by removing all harmonics with $k \leqslant s$ (note that these harmonics are also functions of $w^\top x$, and thus the function that we obtain is also a function of $w^\top x$). That function then has a squared RKHS norm
equal to
$$\sum_{ k \geqslant s, \lambda_k \neq 0}   \| g_k \|_{L_2(\Sb^d)}^2  \lambda_k^{-2}.$$
The summand has an asymptotic equivalent proportional to
$k^{-d -2s-1} k^{d-1} k^{d+2\alpha+1} = k^{d+2\alpha-2s-1}$. Thus if $d+2\alpha-2s \geqslant 0$, the series is divergent (the function is not in the RKHS), i.e., if $s \leqslant \alpha + \frac{d}{2}$.

\paragraph{Best approximation by a function in $\G_2$.}
The squared norm of the $k$-th harmonic
$
\| g_k \|_{L_2(\Sb^d)}^2  $ goes down to zero as $k^{-2s-2}$ and the squared RKHS norm of a  $h $ is equivalent to $\displaystyle \sum_{k \geqslant 0}  \| h_k \|_{L_2(\Sb^d)}^2  k^{d+2\alpha+1}$. Given $\delta$, we may then find the function $h$ such that $\gamma_2(h)^2 
= \sum_{k \geqslant 0}  \| h_k \|_{L_2(\Sb^d)}^2  k^{d+2\alpha+1} \leqslant \delta^2$ with  smallest  $L_2(\Sb^d)$ norm distance to $g$, that is, 
$
 \sum_{k \geqslant 0}  \|g_k -  h_k \|_{L_2(\Sb^d)}^2  .$
The optimal approximation is $h_k = \alpha_k g_k$ for some $\alpha_k \in \rb_+$, with error
$
 \sum_{k \geqslant 0} ( 1 - \alpha_k)^2 \|g_k  \|_{L_2(\Sb^d)}^2
 \sim \sum_{k \geqslant 0} ( 1 - \alpha_k)^2 k^{-2s-2}   $ and squared $\gamma_2$-norm
 $\sum_{k \geqslant 0}   \alpha_k ^2  k^{d+2\alpha+1} k^{-2s-2}
 = \sum_{k \geqslant 0}   \alpha_k ^2  k^{d+2\alpha-2s-1}    $. The optimal $\alpha_k$ is obtained by considering a Lagrange multiplier $\lambda$ such that 
 $
  ( \alpha_k - 1)  k^{-2s-2} + \lambda \alpha_k k^{d+2\alpha-2s-1}   =0
 $, that is,
 $\alpha_k = (  k^{-2s-2} + \lambda   k^{d+2\alpha-2s-1}  ) ^{-1}  k^{-2s-2}
 = (  1 + \lambda   k^{d+2\alpha+1}  ) ^{-1} 
 $.
 We then have
 \BEAS
 \sum_{k \geqslant 0}
  \alpha_k ^2 k^{d+2\alpha-2s-1}
 & \!\!\!= \!\!\!&    \sum_{k \geqslant 0}
  (  1 + \lambda   k^{d+2\alpha+1}  ) ^{-2} 
k^{d+2\alpha-2s-1} \\
& \!\!\!\approx\!\!\! & \int_{0}^\infty \!\! (  1 + \lambda   t^{d+2\alpha+1}  ) ^{-2} 
t^{d+2\alpha-2s-1} dt \mbox{ by approximating a series by an integral},\\
& \!\!\!\propto \!\!\!& \int_{0}^\infty \!\! (  1 + u  ) ^{-2} 
d( t^{d + 2 \alpha - 2s}) \mbox{ with the change of variable } u = \lambda   t^{d+2\alpha+1}\\
& \!\!\!\propto \!\!\!&  \lambda^{-(d + 2 \alpha - 2s)/(d + 2 \alpha + 1)} \mbox{ up to constants},
 \EEAS
 which should be of order $\delta^2$ (this gives the scaling of $\lambda$ as a function of $\delta$).
 Then  the squared error is
\BEAS
  \sum_{k \geqslant 0} ( 1 - \alpha_k)^2 k^{-2s-2} 
& = &     \sum_{k \geqslant 0}
\frac{ \lambda^2  t^{2d+4\alpha+2}}{
 (  1 + \lambda   k^{d+2\alpha+1}  ) ^{2} }
 k^{-2s-2} \\
 & \approx & \int_0^\infty 
 \frac{ \lambda^2  t^{2d+4\alpha-2s}}{
 (  1 + \lambda   t^{d+2\alpha+1}  ) ^{2} }
dt \\
& \approx & \lambda^2 \lambda^{-(2d + 4 \alpha - 2s+1)/(d + 2 \alpha + 1)}
=  \lambda^{-(2d + 4 \alpha - 2s+1-2d-4\alpha-2)/(d + 2 \alpha + 1)}
\\
&= & \lambda^{(2s+1)/(d + 2 \alpha + 1)}
\approx \delta^{-2(2s+1)/(d+2\alpha -2s)}
,
\EEAS
and thus the (non-squared) approximation error scales as 
$
 \delta^{-(2s+1)/(d+2\alpha -2s)}.
$
For  $s=1$, this leads to a scaling as $ \delta^{-3 /(d + 2 \alpha -2)}$.

 \subsection{Proof of Prop.~\ref{prop:linear-sphere}}
 
For $\alpha=1$, by writing $v^\top x = (v^\top x)_+ - (-v^\top x)_+$ we obtain the upperbound $\gamma_1(g) \leqslant 2$. For all other situations, we may compute 
\BEAS
\gamma_2(g) ^2 &  = &  \sum_{k \geqslant 0 } \frac{\| g_k\|_{L_2(\Sb^d)}^2}{\lambda_k^2}.
\EEAS
For $g$ a linear function $g_k=0$ except for $k=1$, for which, we have $g_1(x) = v^\top x$, and thus
$\| g_k\|_{L_2(\Sb^d)}^2 = \int_{\Sb^d} (v^\top x)^2d\tau_d(x)
= v^\top \big( \int_{\Sb^d} xx^\top \tau_d(x) \big) v = 1$. This implies that $\gamma_2(g) = \lambda_1^{-1}$. Given the expression (from Appendix~\ref{app:harmonics}) $\lambda_1 = \frac{d-1}{d}\frac{\alpha}{4\pi}   \frac{\Gamma(\alpha/2) \Gamma(d/2+1)}{\Gamma(\alpha/2+d/2+1)}$ for $\alpha>1$ and $\lambda_1 = \frac{d-1}{2d \pi}$.

\section{Computing $\ell_2$-Haussdorff distance between ellipsoids}
\label{app:ellipsoid}
We assume that we are given two ellipsoids defined as $(x-a)^\top A^{-1} (x-a) \leqslant 1$
and $(x-b)^\top B^{-1} (x-b) \leqslant 1$ and we want to compute their Hausdorff distance. This leads to the two equivalent problems
$$
\max_{\| w\|_2 \leqslant 1 } w^\top ( a- b) - \| B^{1/2} w\|_2 + \| A^{1/2} w\|_2,
$$
$$
\max_{ \| u \|_2 \leqslant 1} \min_{\| v\|_2 \leqslant 1}
 \|  a + A^{1/2} u - b - B^{1/2} v\|_2,
$$
which are related by $w = a + A^{1/2} u - b - B^{1/2} v$. We first review classical methods for optimization of quadratic functions over the $\ell_2$-unit ball.

\paragraph{Minimizing convex quadratic forms over the sphere.}

We consider the following convex optimization problem, with $Q \succcurlyeq 0$; we have by Lagrangian duality:
\BEAS
& & \min_{ \|x\|_2 \leqslant 1} \frac{1}{2} x^\top Q x - q^\top x
\\
& & \max_{ \lambda \geqslant 0 } 
\min_{x \in \rb^d} \frac{1}{2} x^\top Q x - q^\top x + \frac{\lambda}{2} ( \| x\|_2^2 - 1) \\
& & \max_{ \lambda \geqslant 0  } 
-\frac{ 1}{2} q^\top ( Q + \lambda \idm)^{-1} q - \frac{\lambda}{2} \mbox{ with } x = ( Q+ \lambda \idm)^{-1} q.
\EEAS
If $\| Q^{-1} q\|_2 \leqslant 1$, then $\lambda = 0$ and $x = Q^{-1} q$. Otherwise, 
at the optimum, $\lambda >0$ and $ \| x\|_2^2 = q^\top ( Q + \lambda \idm)^{-2} q = 1$, which implies
$ 1 \leqslant \frac{1}{ \lambda  + \lambda_{\min}(Q)} q^\top Q^{-1} q$, which leads to
$ \lambda \leqslant  q^\top Q^{-1} q - \lambda_{\min}(Q) $, which is important to reduce the interval of possible $\lambda$.
The optimal $\lambda$ may then be obtained by binary search (from a single SVD of $Q$).

\paragraph{Minimizing concave quadratic forms over the sphere.}

We consider the following non-convex optimization problem, with $Q \succcurlyeq 0$, for which strong Lagrangian duality is known to hold \citep{boyd}:
\BEAS
& & \min_{ \|x\|_2 \leqslant 1} -\frac{1}{2} x^\top Q x + q^\top x = \min_{ \|x\|_2 = 1} -\frac{1}{2} x^\top Q x + q^\top x
\\
& & \max_{ \lambda \geqslant 0 } 
\min_{x \in \rb^d}  - \frac{1}{2}   x^\top Q x + q^\top x + \frac{\lambda}{2} ( \| x\|_2^2 - 1) \\
& & \max_{     \lambda \geqslant \lambda_{\max}(Q)    } 
-\frac{1}{2} q^\top (  \lambda \idm - Q)^{-1} q - \frac{\lambda}{2} 
\mbox{ with } x = ( Q - \lambda \idm)^{-1} q.
\EEAS
At the optimum, we have $q^\top (  \lambda \idm - Q)^{-2} q = 1$, which implies
$ 1 \leqslant \frac{1}{ [\lambda - \lambda_{\max}(Q)]^2} \|q\|_2^2$, which leads to
$ 0 \leqslant \lambda - \lambda_{\max}(Q) \leqslant  \| q\|_2$. We may perform binary search on $\lambda$ from a single SVD of $Q$.

\paragraph{Computing the Haussdorff distance.}
We need to compute:

\BEAS
& & \max_{ \| u \|_2 \leqslant 1} \min_{\| v\|_2 \leqslant 1}
\frac{1}{2} \|  a + A^{1/2} u - b - B^{1/2} v\|_2^2 \\
& =  & 
\max_{ \| u \|_2 \leqslant 1} \max_{\lambda \geqslant 0} \min_{\| v\|_2 \leqslant 1}
\frac{1}{2} \|  a + A^{1/2} u - b - B^{1/2} v\|_2^2 + \frac{\lambda}{2} ( \|v\|_2^2 - 1)
\\
& =  & 
\max_{ \| u \|_2 \leqslant 1} \max_{\lambda \geqslant 0} - \frac{\lambda}{2}
+   
\frac{1}{2} \|  a -b + A^{1/2} u \|^2 -   \frac{1}{2}  
(a -b + A^{1/2} u)^\top B ( B + \lambda \idm)^{-1} ( a -b + A^{1/2} u)
\\
& =  & 
\max_{ \| u \|_2 \leqslant 1} \max_{\lambda \geqslant 0} - \frac{\lambda}{2}
+ \frac{\lambda}{2}  
(a -b + A^{1/2} u)^\top ( B + \lambda \idm)^{-1} ( a -b + A^{1/2} u)
\EEAS
with $v  = (B+ \lambda \idm)^{-1} B^{1/2} (a -b + A^{1/2} u)$. The interval in $\lambda$ which is sufficient to explore is
$$
\lambda \in \big[ 0, -\lambda_{\min}(B) + \big( \| a - b\|_2^2 + \lambda_{\max}(A^{1/2}) \big)^2 \big],
$$
which are bounds that are independent of $u$.

Given $\lambda \geqslant 0$, we have the problem of 
\BEAS 
&&
\min_{ \mu \geqslant 0 } \max_{ u \in \rb^d} 
\frac{\lambda}{2}  
(a -b + A^{1/2} u)^\top ( B + \lambda \idm)^{-1} ( a -b + A^{1/2} u)
- \frac{\mu}{2} ( \| u\|_2^2 - 1) - \frac{\lambda}{2}
\\
& = &
\min_{ \mu \geqslant 0 } \max_{ u \in \rb^d} 
\frac{\lambda}{2}  
(a -b  )^\top ( B + \lambda \idm)^{-1} ( a -b  )
+ \frac{\mu - \lambda}{2}
+ \lambda u^\top A^{1/2} ( B + \lambda \idm)^{-1} ( a -b )  \\
& & \hspace*{5cm}
- \frac{1}{2} u^\top \big(
\mu \idm -  \lambda A^{1/2} (B+\lambda\idm)^{-1} A^{1/2}
\big) u 
\\
& = &
\min_{ \mu \geqslant 0 }  
\frac{\lambda}{2}  
(a -b  )^\top ( B + \lambda \idm)^{-1} ( a -b  )
+ \frac{\mu -  \lambda}{2}
\\
& & \hspace*{0cm}
+ \lambda^2 ( a -b )^\top  ( B + \lambda \idm)^{-1} A^{1/2} \big(
\mu \idm - \lambda A^{1/2} (B+\lambda\idm)^{-1} A^{1/2}
\big) ^{-1} A^{1/2}  ( B + \lambda \idm)^{-1} ( a -b )
\\
\EEAS
We have
$\displaystyle
u = ( \frac{\mu}{\lambda} \idm - A^{1/2} ( B+\lambda \idm)^{-1} A^{1/2})^{-1} A^{1/2} ( B+\lambda \idm)^{-1} ( a- b),
$
leading to 
$ \displaystyle
w \propto ( \lambda^{-1} B - \mu^{-1} A + \idm) ( a - b)$.
We need $\frac{\mu}{\lambda} \geqslant   \lambda_{\max}(    A^{1/2} (B+\lambda\idm)^{-1} A^{1/2} )$.
Moreover
$$0 \leqslant \frac{\mu}{\lambda}  - 
\lambda_{\max}\big(    A^{1/2} (B+\lambda\idm)^{-1} A^{1/2} \big)  \leqslant 
\big\|
A^{1/2} ( B + \lambda \idm)^{-1} ( a- b)
\big\|.
$$

This means that the $\ell_2$-Haussdorff distance may be computed by solving in $\lambda$ and $\mu$, by exhaustive search with respect to $\lambda$ and by binary search (or Newton's method) for $\mu$.
The complexity of each iteration is that of a singular value decomposition, that is $O(d^3)$. For more details on optimization of quadratic functions on the unit-sphere, see~\citet{forsythe1965stationary}.

\acks{The author was partially supported by the European Research Council (SIERRA Project),
and thanks Nicolas Le Roux for helpful discussions. The author also thanks Varun Kanade for pointing the NP-hardness linear classification result. The main part of this work was carried through while visiting the Centre de Recerca Matem\`atica (CRM) in Barcelona.}

\end{document}